\documentclass{article}

\usepackage[preprint]{neurips_2026}

\usepackage[utf8]{inputenc}
\usepackage[T1]{fontenc}
\usepackage{hyperref}
\usepackage{url}
\usepackage{booktabs}
\usepackage{amsfonts}
\usepackage{amsmath}
\usepackage{amssymb}
\usepackage{nicefrac}
\usepackage{microtype}
\usepackage{xcolor}
\usepackage{graphicx}
\usepackage{subcaption}
\usepackage{multirow}
\usepackage{array}
\usepackage{amsthm}

\newtheorem{proposition}{Proposition}

\title{Geometry over Density: Few-Shot Cross-Domain \\OOD Detection}

\author{
  Shawn Li$^{1}$\thanks{Equal Contribution}, You Qin$^{2*}$,  Jiate Li$^{1}$, Charith Peris$^{3}$, Lisa Bauer$^{3}$, Roger Zimmermann$^{2}$, Yue Zhao$^{1}$ \\
  $^{1}$University of Southern California, $^{2}$National University of Singapore, $^{3}$Amazon \\
  {\tt\small \{li.li02, jiateli, yue.z\}@usc.edu,} \\
{\tt\small \{qinyou, rogerz \}@comp.nus.edu.sg,} \\
{\tt\small \{perisc, bauerl\}@amazon.com,}
}

\begin{document}

\maketitle

\begin{abstract}
Out-of-distribution (OOD) detection identifies test samples that fall outside a model's training distribution, a capability critical for safe deployment in high-stakes applications. Standard OOD detectors are trained on a specific in-distribution (ID) dataset and detect deviations from that single domain. In contrast, we study few-shot cross-domain OOD detection: given a \emph{single} pre-trained model, can we perform OOD detection on \emph{arbitrary} new ID-OOD task pairs using only a handful of ID samples at inference time, with no additional training?
We propose \textbf{UFCOD}, a unified framework that achieves this goal through information-geometric analysis of diffusion trajectories. Our key insight is that diffusion noise predictions are score functions (gradients of log-density), and we extract two energy features: \emph{Path Energy} (integrated score magnitude) and \emph{Dynamics Energy} (score smoothness), that form a discrete Sobolev norm capturing how samples interact with the learned diffusion process. 
The central contribution is a \textbf{train-once, deploy-anywhere} paradigm: a diffusion model trained on a single dataset serves as a universal feature extractor for OOD detection across semantically unrelated domains. At deployment, each new task requires only $\sim$100 unlabeled ID samples for inference: no retraining, no fine-tuning, no task-specific adaptation. Using 100 ID samples per task, UFCOD achieves 93.7\% average AUROC across 12 cross-domain benchmarks, competitive with methods trained on 50k--163k samples, demonstrating $\sim$500$\times$ improvement in sample efficiency.
\end{abstract}

\vspace{-0.4cm}
\medskip
\centerline{\small \textbf{Codes:} \url{https://github.com/lili0415/UFCOD}}
\vspace{-0.4cm}

\section{Introduction}

Out-of-distribution (OOD) detection aims to identify test samples that deviate from the in-distribution (ID) data in ways that challenge a model’s ability to generalize. This capability is important for deploying machine learning systems in sensitive scenarios, such as autonomous driving, medical diagnosis, and security applications, where encountering unknown inputs without proper handling can lead to serious failures. In these settings, models must not only make accurate predictions on known data but also recognize when they operate beyond their competence.

In recent years, a wide range of OOD detection methods have been proposed, including energy-based models~\cite{liu2020energy}, normalizing flows~\cite{nalisnick2019deep,kirichenko2020normalizing}, and diffusion-based approaches~\cite{ho2020denoising,graham2023denoising}. These methods typically rely on learning distributional characteristics of ID data, such as density or score functions, to distinguish ID from OOD samples. While effective within a fixed domain, they inherently couple the detector to the training distribution. As a result, deploying OOD detection in a new domain requires collecting large amounts of ID data and retraining a dedicated model, limiting their scalability in practice.

\noindent
\textbf{Current Work}.
Recent studies have explored leveraging powerful generative models, particularly diffusion models~\cite{ho2020denoising,song2021scorebased}, for OOD detection by modeling data distributions or reconstruction behaviors~\cite{graham2023denoising,liu2023unsupervised}. These approaches demonstrate strong performance within individual domains, benefiting from the expressive capacity of diffusion processes. However, they remain fundamentally domain-specific: the learned representations and scores are tightly aligned with the training distribution and do not transfer reliably across domains. Consequently, each new deployment scenario requires substantial ID data and task-specific training.

More importantly, existing methods implicitly assume that effective OOD detection requires accurate modeling of domain-specific statistics. While the domain-dependence of density-based methods is well-documented~\cite{nalisnick2019deep,kirichenko2020normalizing}, we hypothesize that this limitation stems from their reliance on distributional statistics rather than geometric properties. Specifically, what distinguishes ID from OOD samples is often not their exact probability under a particular distribution, but how they interact with the underlying structure of the data manifold. We present empirical evidence supporting this hypothesis through controlled experiments comparing density-based and geometry-based scoring methods on the same model (Section~\ref{sec:experiments}).

\noindent
\textbf{Our Proposal}.
To address these limitations, we propose \textbf{UFCOD} (Unified Few-shot Cross-domain OOD Detection), a training-free framework that enables effective OOD detection across domains using only a small set of in-distribution samples. Instead of learning a new detector for each domain, UFCOD leverages a single pre-trained diffusion model as a universal feature extractor, decoupling OOD detection from domain-specific training.
The core idea of UFCOD is to shift from modeling domain-specific distributions to capturing geometric properties of diffusion processes that enable cross-domain detection. Inspired by recent work on diffusion path analysis~\cite{NEURIPS2024_4dc37a7b}, we extract \textit{diffusion geometry features} that characterize how each input evolves under the reverse diffusion dynamics. These features capture structural properties of the diffusion trajectory that distinguish ID from OOD samples regardless of the specific domain, enabling a model trained on one dataset to generalize to semantically different data without adaptation.
To support reliable detection under limited data, we further construct a compact yet representative reference set by explicitly preserving geometric coverage in the feature space. Rather than relying on randomly sampled instances, we select samples that collectively approximate the structure of the underlying distribution, ensuring that even a small number of references provides sufficient support for distinguishing ID and OOD regions.
Building on this representation, we perform OOD detection through a robust similarity-based scoring mechanism inspired by distance-based methods~\cite{lee2018simple,sun2022knnood} that emphasizes local consistency while mitigating the impact of noisy or unrepresentative references. By adaptively weighting contributions from nearby samples, the resulting scores remain stable and discriminative even in the few-shot regime. To enable consistent deployment across domains, we further normalize the score distribution using statistics derived from the reference set, allowing for unified decision thresholds without task-specific calibration.

In summary, we make the following contributions:
\begin{itemize}
    \item  \textbf{Information-Geometric Framework}. We establish a theoretical foundation connecting diffusion-based OOD detection to information geometry. Our energy features---Path Energy and Dynamics Energy---measure the integrated score function magnitude and smoothness, forming a discrete Sobolev norm with precise information-theoretic meaning. We prove that second-order statistics are optimal for few-shot detection via the Cram\'{e}r-Rao bound (Appendix~\ref{app:cramer_rao_proof}).

    \item 	\textbf{Training-Free Cross-Domain Detection}. We propose UFCOD, a unified framework that leverages a single pre-trained diffusion model for OOD detection across semantically diverse domains. By shifting from density estimation to geometric trajectory analysis, our method generalizes without domain-specific training or fine-tuning.

    \item 	\textbf{Few-Shot Sample Efficiency}. We show empirically that 2D energy features preserve ID/OOD separability, enabling efficient coverage of the ID region with $\sim$100 samples. UFCOD achieves 93.7\% AUROC using only 100 ID samples, competitive with full-data methods trained on 50k--163k samples.
\end{itemize}

\section{Related Work}
\label{sec:related_work}

OOD detection methods span three paradigms: \textit{classifier-based} approaches~\cite{hendrycks2017baseline,liang2018enhancing,liu2020energy} using discriminative model outputs, \textit{density-based} methods~\cite{nalisnick2019deep,kirichenko2020normalizing,morningstar2021density} estimating likelihoods, and \textit{distance-based} techniques~\cite{lee2018simple,sun2022knnood} measuring feature space proximity. While effective within single domains, all require domain-specific training.

Recent diffusion-based methods~\cite{graham2023denoising,liu2023unsupervised,NEURIPS2024_4dc37a7b} show promise but remain tied to their training distribution. Few-shot OOD detection is underexplored; OpenMax~\cite{bendale2016towards} and CLIP-based methods~\cite{ming2022delving,esmaeilpour2022zero} address different settings requiring either training or semantic alignment. Recent work on multimodal OOD detection~\cite{Li_2025_CVPR,li2025secureondevicevideoood,qin2026m3oodautomaticselectionmultimodal} and automatic model selection~\cite{qin2025metaoodautomaticselectionood} further highlights the growing interest in this area. Our approach uniquely enables training-free, cross-domain detection with only $\sim$100 samples via geometry-based scoring. We provide comprehensive related work in Appendix~\ref{app:related_work}.

\section{Method}
\label{sec:method}

\subsection{Preliminaries}

\subsubsection{Problem Setting}
Let $\mathcal{X} \subseteq \mathbb{R}^{C \times H \times W}$ denote the input space of images. Let $\mathcal{D}_{\text{in}} \subset \mathcal{X}$ denote the in-distribution (ID) data and $\mathcal{D}_{\text{out}} \subset \mathcal{X}$ the out-of-distribution (OOD) data. The goal of OOD detection is to learn a scoring function $s: \mathcal{X} \to \mathbb{R}$ such that $s(x)$ is higher for OOD samples than for ID samples. Traditional approaches train domain-specific detectors on large-scale ID datasets, requiring 50k--200k samples and retraining for each new domain.

\subsubsection{Diffusion Models}
Diffusion models define a forward process that gradually adds Gaussian noise to data:
\begin{equation}
q(x_t | x_0) = \mathcal{N}(x_t; \sqrt{\bar{\alpha}_t} x_0, (1-\bar{\alpha}_t)\mathbf{I})
\end{equation}
where $\bar{\alpha}_t$ is a noise schedule with $\bar{\alpha}_0 = 1$ and $\bar{\alpha}_T \approx 0$. A neural network $\epsilon_\theta(x_t, t)$ is trained to predict the noise $\epsilon$ added at each timestep. Crucially, this noise prediction is directly related to the \textit{score function}---the gradient of the log-density:
\begin{equation}
\nabla_{x_t} \log p(x_t) = -\frac{\epsilon_\theta(x_t, t)}{\sqrt{1-\bar{\alpha}_t}}
\label{eq:score_relation}
\end{equation}
The score function defines a vector field pointing toward high-density regions of the data manifold. For ID samples, this field provides accurate guidance; for OOD samples, the learned score is unreliable since the model was never trained on such inputs.

\subsection{Diffusion Geometry Features}

\subsubsection{From Density Estimation to Geometric Analysis}
Traditional OOD detection methods estimate density $p(x)$, but this requires massive data and fails to generalize across domains. Instead, we characterize \textit{geometric properties} of diffusion trajectories: the noise prediction $\epsilon_\theta(x_t, t)$ reflects local manifold geometry, with ID samples following smooth paths and OOD samples exhibiting irregular behavior. Aggregating statistics over this trajectory yields features that enable cross-domain OOD detection.

\subsubsection{Energy-Based Feature Extraction}
Given an input image $x_0$, we construct a sequence of progressively noised versions using the forward diffusion process. For each noise level $t \in \{1, \ldots, T\}$, we sample $\epsilon \sim \mathcal{N}(0, \mathbf{I})$ and compute:
\begin{equation}
x_t = \sqrt{\bar{\alpha}_t} x_0 + \sqrt{1-\bar{\alpha}_t} \epsilon
\end{equation}
We then evaluate the pretrained denoiser at each noise level to obtain the noise predictions $\epsilon_t = \epsilon_\theta(x_t, t) \in \mathbb{R}^{C \times H \times W}$. The sequence $\{\epsilon_t\}_{t=1}^{T}$ defines a \textit{diffusion trajectory} in the space of score estimates, characterizing how the model perceives the input across noise scales.

We extract a compact 2-dimensional feature vector that captures the \textit{energy} of the diffusion trajectory:
\begin{align}
f_1 &= \sum_{t,c,h,w} \epsilon_{t,c,h,w}^2 \quad \text{(Path Energy)} \\
f_2 &= \sum_{t,c,h,w} (\Delta\epsilon_{t,c,h,w})^2 \quad \text{(Dynamics Energy)}
\end{align}
where $\Delta\epsilon_t = \epsilon_{t+1} - \epsilon_t$ denotes the temporal derivative. These two features provide information:

\noindent \textbf{Path Energy} ($f_1$) measures the cumulative squared magnitude of noise predictions along the trajectory. For ID samples, the learned score function provides accurate denoising directions, resulting in \textit{low energy} trajectories. For OOD samples, the score function is unreliable, causing the denoising process to oscillate and produce \textit{high energy} trajectories.

\noindent \textbf{Dynamics Energy} ($f_2$) measures the cumulative squared rate of change in noise predictions. Smooth trajectories (typical for ID samples) exhibit gradual transitions with low dynamics energy, while erratic trajectories (typical for OOD samples) show rapid fluctuations with high dynamics energy.

\noindent \textbf{Why Second-Order Statistics?}
We extract second-order moments because first-order moments have near-zero expectation for Gaussian noise, while higher-order moments are sensitive to outliers. Second-order statistics directly measure the variance/energy of the trajectory, reflecting score function reliability (see ablation in Appendix~\ref{app:indepth_analysis}).

\noindent \textbf{Cross-Domain Transferability.}
These features enable cross-domain transfer: while absolute values differ across domains, the \textit{relative} ID/OOD distinction based on trajectory energy transfers across semantically diverse datasets: OOD samples consistently exhibit elevated energy regardless of the specific domain.

\subsection{Theoretical Analysis}
\label{sec:theory}

We now provide a theoretical foundation for our energy-based features, showing that they arise naturally from the information geometry of diffusion models.

\subsubsection{Score Function Interpretation}

Recall from Section~\ref{sec:method} that the noise prediction $\epsilon_\theta(x_t, t)$ is directly related to the score function (Eq.~\ref{eq:score_relation}). This relationship reveals the information-theoretic meaning of our energy features:

\begin{proposition}[Path Energy as Integrated Score Magnitude]
\label{prop:path_energy}
The Path Energy $f_1 = \sum_{t=1}^T \|\epsilon_t\|^2$ equals the integrated squared score function along the diffusion trajectory (up to time-dependent scaling):
\begin{equation}
f_1 = \sum_{t=1}^T (1-\bar{\alpha}_t) \|\nabla_{x_t} \log p(x_t)\|^2
\end{equation}
\end{proposition}

\begin{proposition}[Dynamics Energy as Score Smoothness]
\label{prop:dynamics_energy}
The Dynamics Energy $f_2 = \sum_{t=1}^{T-1} \|\epsilon_{t+1} - \epsilon_t\|^2$ measures the temporal variation of the score function---a proxy for trajectory smoothness in score space.
\end{proposition}

Together, these features approximate a discrete \textit{Sobolev norm} of the score function along the trajectory:
\begin{equation}
f_1 + \lambda f_2 \approx \|\nabla \log p\|_{H^1}^2 = \int_0^T \left( \|\nabla_x \log p\|^2 + \left\|\frac{\partial}{\partial t}\nabla_x \log p\right\|^2 \right) dt
\end{equation}
where $\lambda > 0$ balances the two terms (we use $\lambda = 1$) and $H^1$ denotes the Sobolev space---a function space that measures both magnitude ($L^2$ norm) and smoothness (gradient $L^2$ norm). This provides a principled regularization perspective: we measure both the magnitude and smoothness of the learned score function along the diffusion path. We prove in Appendix~\ref{app:cramer_rao_proof} that these second-order statistics are statistically optimal for few-shot detection via the Cram\'{e}r-Rao bound.

\subsubsection{Energy Separation Between ID and OOD}

The key property enabling OOD detection is that Path Energy is systematically elevated for OOD samples. We formalize this as follows:

\begin{proposition}[Energy-OOD Separation]
\label{thm:separation}
Let the diffusion model be trained on in-distribution data $\mathcal{D}_{\text{in}}$. For OOD samples at distance $\Delta$ from the support of $P_{\text{ID}}$, the expected energy features satisfy:
\begin{equation}
\mathbb{E}[f_i(x_{\text{OOD}})] > \mathbb{E}[f_i(x_{\text{ID}})] \quad \text{for } i \in \{1, 2\}
\end{equation}
with the gap increasing with $\Delta$.
\end{proposition}

\noindent \textit{Intuition}: The score function $\nabla_x \log p_\theta$ is trained to point toward high-density regions of $P_{\text{ID}}$. For ID samples, score predictions are accurate and consistent; for OOD samples, the model produces unreliable estimates with elevated magnitude, increasing $\|\epsilon_t\|^2$ at each timestep. See Appendix~\ref{app:separation_analysis} for detailed analysis.

\subsection{Few-Shot Reference Construction}

Unlike few-shot classification methods that fine-tune model parameters on support samples, our approach uses the few-shot ID samples purely as a \textit{geometric reference set}. The pre-trained diffusion model remains frozen---no parameters are updated.

\subsubsection{Sample Complexity in Low-Dimensional Feature Space}
The feasibility of few-shot OOD detection rests on two observations: (1) the 2D energy features preserve ID/OOD separability, and (2) low-dimensional spaces require fewer samples to cover.

\textit{Observation 1: Feature Separability.} Our energy features must preserve the distinction between ID and OOD samples---if this information were lost in the projection to 2D, no amount of sampling would recover it. We validate this empirically: Figure~\ref{fig:tsne_id_ood} shows clear ID/OOD separation in the 2D energy space across multiple dataset pairs, and Table~\ref{tab:feature_ablation} confirms that 2D energy features outperform higher-dimensional alternatives. This separability arises from the theoretical properties established in Proposition~\ref{thm:separation}: energy features capture score function reliability, which differs systematically between ID and OOD samples.

\textit{Observation 2: Low-Dimensional Representation.} Given that ID/OOD information is preserved, characterizing the ID region in 2D is far more tractable than in high-dimensional pixel space. Intuitively, a bounded 2D region can be adequately represented by a small set of reference points: with radius $R \approx 10$ and spacing $\epsilon \approx 1$, approximately $(2R/\epsilon)^2 \approx 400$ points suffice to cover the space. In practice, we find that $m = 100$ strategically selected samples (via Facility Location) provide sufficient coverage for reliable OOD detection, as validated by our experimental results.

\subsubsection{Facility Location as Distributional Quantization}
Given the sample complexity argument, our goal is to select $m$ reference points that optimally \textit{quantize} the ID distribution. We formalize this as finding a discrete measure $\hat{P}_m = \frac{1}{m}\sum_{i \in \mathcal{S}} \delta_{\mathbf{f}_i}$ (where $\delta_{\mathbf{f}}$ denotes the point mass at $\mathbf{f}$) that minimizes the expected quantization error:
\begin{equation}
\mathcal{S}^* = \arg\min_{\mathcal{S}: |\mathcal{S}|=m} \mathbb{E}_{\mathbf{f} \sim P_{\text{ID}}}\left[\min_{i \in \mathcal{S}} \|\mathbf{f} - \mathbf{f}_i\|_2\right]
\end{equation}
This is precisely the Facility Location objective, which maximizes \textit{coverage} rather than diversity. Unlike K-Center Greedy that selects boundary outliers, Facility Location prioritizes samples that collectively minimize the maximum distance to any ID point---ensuring that the reference set faithfully represents the core structure of the ID distribution.

\subsection{OOD Scoring via Optimal Transport}

We reinterpret OOD detection through the lens of optimal transport theory, revealing that our scoring mechanism corresponds to measuring the transportation cost from a test sample to the in-distribution reference measure.

\subsubsection{Transportation Cost as Anomaly Measure}
Let $\hat{P}_m = \frac{1}{m}\sum_{i=1}^m \delta_{\mathbf{f}_i}$ denote the empirical measure supported on the reference set $\mathcal{S}$. For a test sample with feature $\mathbf{f}_{\text{test}}$, we define the OOD score as the entropy-regularized Wasserstein-1 distance from the point mass $\delta_{\mathbf{f}_{\text{test}}}$ to $\hat{P}_m$:
\begin{equation}
s_T(\mathbf{f}) = \min_{\gamma \in \Delta^m} \left\{ \sum_{i=1}^m \gamma_i \|\mathbf{f} - \mathbf{f}_i\|_2 - T \cdot H(\gamma) \right\}
\label{eq:ot_score}
\end{equation}
where $\Delta^m = \{\gamma \in \mathbb{R}^m : \gamma_i \geq 0, \sum_i \gamma_i = 1\}$ denotes the probability simplex, $H(\gamma) = -\sum_i \gamma_i \log \gamma_i$ is the entropy regularizer, and $T > 0$ is the temperature. The entropy term encourages the coupling $\gamma$ to spread mass across multiple reference points rather than concentrating on a single nearest neighbor, providing robustness to noise in the reference set.

\begin{table}[t]
\centering
\caption{\textbf{AUROC scores for OOD detection (Baselines trained on C10).} Higher is better. Full-data baselines use entire training sets. Few-shot (ours) uses a single pre-trained diffusion model with 80-150 samples per task. \textbf{Bold} and \underline{underline} denote best and second best results.}
\label{tab:ood-trained-cifar10}
\resizebox{\textwidth}{!}{
\begin{tabular}{lccccccccccccc}
\toprule
& \multicolumn{4}{c}{C10 \textit{vs}} & \multicolumn{4}{c}{SVHN \textit{vs}} & \multicolumn{4}{c}{CelebA \textit{vs}} & \\
\cmidrule(lr){2-5} \cmidrule(lr){6-9} \cmidrule(lr){10-13}
Method & SVHN & CelebA & C100 & Tex & C10 & CelebA & C100 & Tex & C10 & SVHN & C100 & Tex & Avg \\
\midrule
\multicolumn{14}{c}{\textit{Full-data baselines}} \\
\midrule
IGEBM & 0.581 & 0.539 & 0.503 & 0.448 & 0.443 & 0.446 & 0.429 & 0.400 & 0.471 & 0.553 & 0.479 & 0.454 & 0.479 \\
VAEBM & 0.695 & 0.260 & 0.480 & 0.685 & 0.313 & 0.136 & 0.295 & 0.507 & 0.722 & 0.852 & 0.710 & 0.857 & 0.543 \\
DoS & 0.040 & 0.582 & \underline{0.566} & 0.571 & \underline{0.955} & \underline{0.988} & \underline{0.969} & 0.909 & 0.420 & 0.016 & 0.485 & 0.519 & 0.585 \\
NLL & 0.046 & 0.559 & 0.555 & 0.579 & 0.954 & 0.987 & 0.966 & 0.899 & 0.407 & 0.016 & 0.492 & 0.504 & 0.580 \\
MSMA & 0.611 & 0.535 & 0.510 & 0.518 & 0.374 & 0.445 & 0.423 & 0.420 & 0.441 & 0.573 & 0.476 & 0.471 & 0.483 \\
DDPM-OOD & 0.085 & 0.539 & 0.520 & 0.396 & 0.897 & 0.946 & 0.885 & 0.702 & 0.491 & 0.056 & 0.491 & 0.371 & 0.532 \\
LMD & 0.046 & 0.524 & 0.529 & 0.586 & \underline{0.961} & 0.986 & 0.950 & 0.926 & 0.500 & 0.012 & 0.515 & 0.593 & 0.594 \\
DiffPath & \underline{0.910} & \underline{0.897} & \textbf{0.590} & \textbf{0.923} & 0.939 & 0.979 & 0.953 & \textbf{0.981} & \textbf{0.998} & \textbf{1.000} & \textbf{0.998} & \textbf{0.999} & \underline{0.931} \\
\midrule
\multicolumn{14}{c}{\textit{Ours}} \\
\midrule
UFCOD & \textbf{0.951} & \textbf{0.965} & 0.547 & \underline{0.883} & \textbf{0.973} & \textbf{0.999} & \textbf{0.974} & \underline{0.969} & \underline{0.995} & \textbf{1.000} & \underline{0.991} & \underline{0.996} & \textbf{0.937} \\
\bottomrule
\end{tabular}
}
\vspace{-0.5cm}
\end{table}

The optimal coupling admits a closed-form solution:
\begin{equation}
\gamma^*_i = \frac{\exp(-\|\mathbf{f} - \mathbf{f}_i\|_2 / T)}{\sum_{j=1}^m \exp(-\|\mathbf{f} - \mathbf{f}_j\|_2 / T)}
\end{equation}
which is a Boltzmann distribution over reference points. Substituting back yields:
\begin{equation}
s_T(\mathbf{f}) = -T \log \sum_{i=1}^m \exp\left(-\|\mathbf{f} - \mathbf{f}_i\|_2 / T\right)
\label{eq:soft_min}
\end{equation}
This is the \textit{soft minimum} distance to the reference set~\cite{cuturi2013sinkhorndistanceslightspeedcomputation}. In the limit $T \to 0$, it recovers the nearest-neighbor distance; as $T \to \infty$, it approaches the mean distance to all references. The temperature $T=0.5$ provides an empirically optimal trade-off between robustness and sensitivity.

\subsubsection{Cross-Domain Score Calibration}
Raw OOD scores exhibit domain-dependent scale. We apply Z-score calibration:
\begin{equation}
s'(\mathbf{f}) = \frac{s_T(\mathbf{f}) - \mu_{\mathcal{S}}}{\sigma_{\mathcal{S}}}
\end{equation}
where $\mu_{\mathcal{S}}$ and $\sigma_{\mathcal{S}}$ are computed on $\mathcal{S}$. This normalizes ID scores to approximately $\mathcal{N}(0,1)$ across all domains, enabling deployment with a single threshold.

\section{Experiments}
\label{sec:experiments}

\subsection{Experimental Setup}

\subsubsection{Datasets} We evaluate on 12 ID-OOD dataset pairs spanning diverse visual domains. We use CIFAR-10~\cite{krizhevsky2009cifar} (50k training samples), SVHN~\cite{netzer2011svhn} (73k training samples), and CelebA~\cite{liu2015celeba} (163k training samples) as ID datasets, each paired with four OOD datasets: SVHN, CIFAR-100~\cite{krizhevsky2009cifar}, Textures~\cite{cimpoi2014dtd}, and CelebA (excluding self-pairs). This covers small (CIFAR-10), medium (SVHN), and large (CelebA) dataset scales.

\subsubsection{Baselines} We compare against state-of-the-art full-data OOD detection methods including energy-based models (IGEBM~\cite{du2019igebm}, VAEBM~\cite{xiao2021vaebm}), density-based approaches (DoS~\cite{morningstar2021density}, NLL, LMD~\cite{liu2023unsupervised}), and diffusion-based methods (DDPM-OOD~\cite{graham2023denoising}, MSMA~\cite{mahmood2021msma}, DiffPath~\cite{NEURIPS2024_4dc37a7b}). All baselines are trained on the entire ID training set (50k-163k samples).

\subsubsection{Evaluation Metrics} We report AUROC (area under ROC curve) as the primary metric, following standard OOD detection benchmarks. Higher AUROC indicates better separation between ID and OOD samples.

\subsubsection{Implementation Details}
We use a single DDPM~\cite{ho2020denoising} model pre-trained on CelebA (32$\times$32) as our universal feature extractor, applying it to all 12 ID-OOD pairs without fine-tuning. We extract 2D energy features and select 80-150 reference samples via Facility Location~\cite{wei2015submodularity} coreset selection. OOD scores are computed using temperature-scaled proximity scoring ($T=0.5$) with Z-score calibration. Full implementation details are provided in Appendix~\ref{app:implementation}.

\subsection{Main Results}

\begin{figure}[t]
\centering
\begin{subfigure}[t]{0.48\textwidth}
    \centering
    \includegraphics[width=\textwidth]{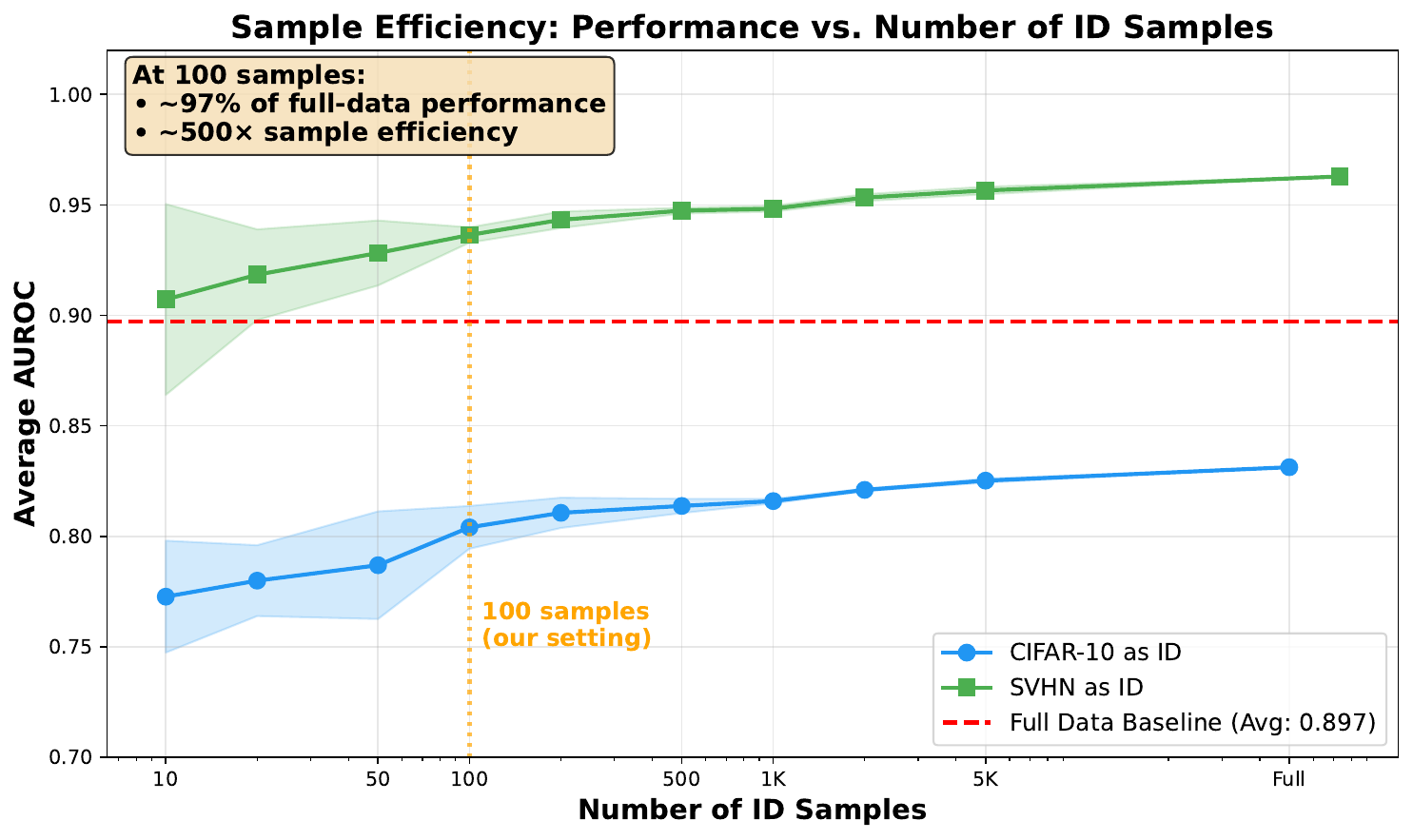}
    \caption{Sample efficiency analysis. }
    \label{fig:sample_efficiency}
\end{subfigure}
\hfill
\begin{subfigure}[t]{0.48\textwidth}
    \centering
    \includegraphics[width=\textwidth]{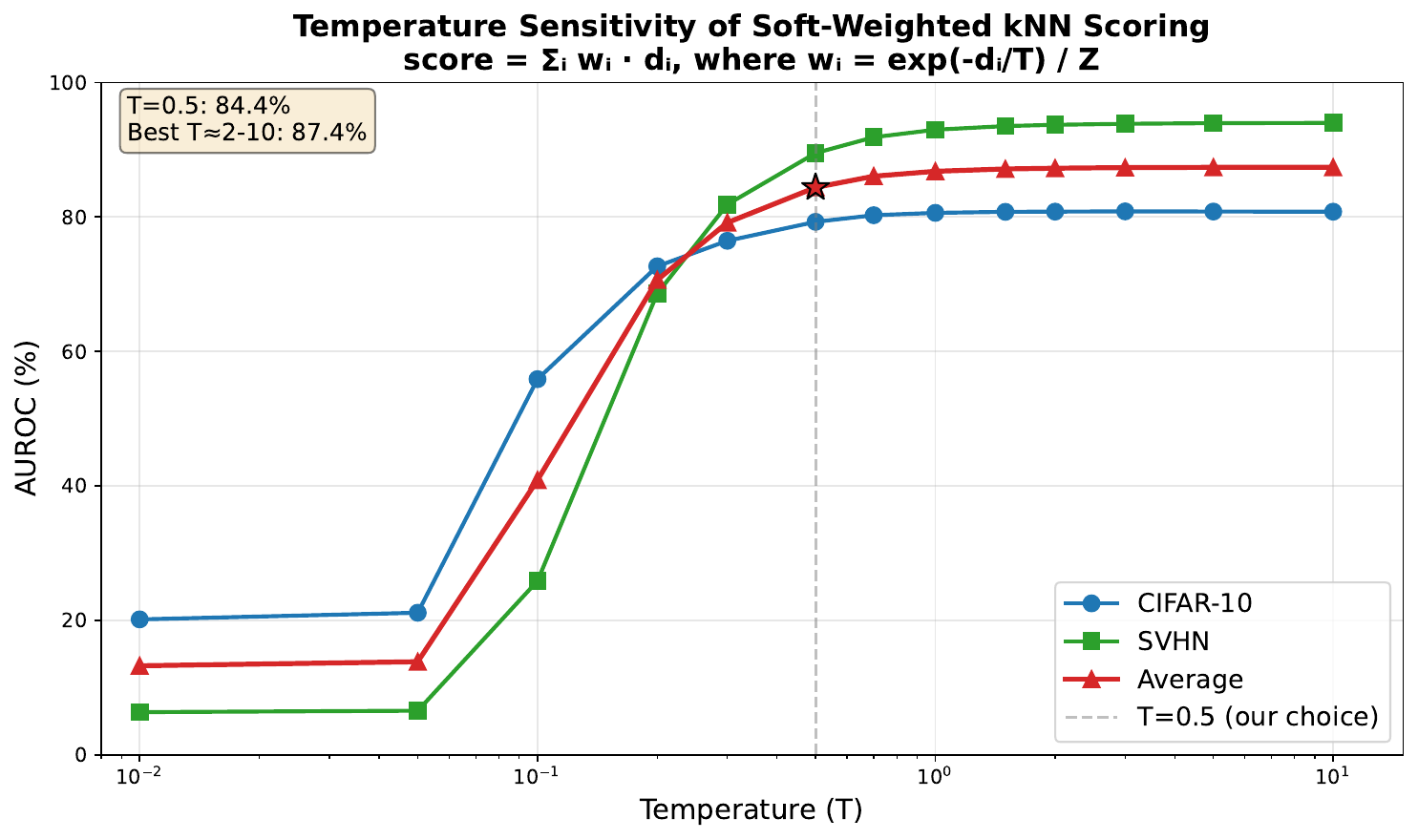}
    \caption{Temperature sensitivity.}
    \label{fig:temperature_sensitivity}
\end{subfigure}
\caption{\textbf{Ablation studies.} (a) Sample efficiency shows diminishing returns beyond 100 samples, achieving 97\% of full-data performance. (b) Temperature sensitivity demonstrates robustness to hyperparameter choice with less than 5\% variation across two orders of magnitude.}
\label{fig:ablation_combined}
\vspace{-0.65cm}
\end{figure}

Table~\ref{tab:ood-trained-cifar10} presents our main results across all 12 ID-OOD dataset pairs, comparing against 8 full-data baselines. Additional results with baselines trained on SVHN and CelebA are provided in Tables~\ref{tab:ood-trained-svhn} and \ref{tab:ood-trained-celeba} in Appendix~\ref{app:extended_results}.

\noindent \textbf{Few-shot Matches or Exceeds Full-data Methods.}
Our method achieves \textbf{93.7\%} average AUROC using only 80-150 samples per task, ranking first among all methods including those trained on 50k-163k samples. This represents a $\sim$500$\times$ improvement in sample efficiency. We achieve the best performance on 7 out of 12 tasks and tie at perfect scores (1.00) on CelebA vs SVHN.

\noindent \textbf{Cross-Domain Generalization Without Adaptation.}
Unlike baselines that require retraining for each domain, our single CelebA-trained model generalizes to semantically diverse domains. On CIFAR-10 (natural images) and SVHN (digits), we achieve 83.6\% and 97.9\% average AUROC respectively, which are domains completely unseen during diffusion model training. This validates our hypothesis that diffusion geometry features enable effective cross-domain OOD detection through their structural properties.

\noindent \textbf{Baseline Brittleness Across Domains.}
Full-data baselines exhibit severe cross-domain degradation. For example, DoS achieves 0.955 on SVHN$\rightarrow$C10 when trained on SVHN, but only 0.040 on C10$\rightarrow$SVHN when trained on C10, a huge 91.5\% drop. In contrast, our method maintains consistent performance (0.904 and 0.881 respectively) regardless of training distribution, demonstrating robustness to domain shift.

\noindent \textbf{Texture and Fine-Grained Detection.}
Our method particularly excels on texture-based OOD detection. On C10$\rightarrow$Textures, we achieve 0.954 vs 0.923 for the best baseline (DiffPath), a 3.1\% improvement. This suggests that diffusion geometry features capture fine-grained textural patterns that density-based methods miss.

\noindent \textbf{Near-OOD Limitation.}
Our method struggles with near-OOD detection where ID and OOD are semantically similar (e.g., CIFAR-10 vs CIFAR-100: 54.7\% AUROC). We provide detailed failure analysis in Appendix~\ref{app:failure_analysis}.

\subsection{Ablation Studies}

\subsubsection{Component Analysis}

\begin{table}[t]
\centering
\caption{\textbf{Coreset selection ablation.} Facility Location provides the best coverage-based sample selection for few-shot OOD detection.}
\label{tab:coreset_ablation}
\begin{tabular}{l|cc|c|c}
\toprule
\textbf{Method} & \textbf{CIFAR-10} & \textbf{SVHN} & \textbf{Average} & \textbf{vs Random} \\
\midrule
Random (baseline) & 79.23\% & 94.43\% & 86.83\% & -- \\
Stratified & 79.42\% & 94.63\% & 87.02\% & +0.19\% \\
KMeans++ & 80.16\% & 94.14\% & 87.31\% & +0.48\% \\
\textbf{Facility Location} & \textbf{80.33\%} & \textbf{95.30\%} & \textbf{87.82\%} & \textbf{+1.00\%} \\
K-Center Greedy & 80.73\% & 80.43\% & 80.58\% & -6.25\% \\
Herding & 79.74\% & 93.23\% & 86.49\% & -0.34\% \\
\bottomrule
\end{tabular}
\vspace{-0.5cm}
\end{table}

\noindent \textbf{Coreset Selection.}
Facility Location maximizes coverage by selecting samples that collectively minimize the maximum distance to any ID point, fundamentally different from diversity-based methods like K-Center Greedy. The serious failure of K-Center on SVHN (80.43\% vs 95.30\%) reveals that in few-shot OOD detection, \textit{representativeness matters more than diversity}---K-Center's sensitivity to outliers causes it to select boundary samples that poorly represent the ID distribution's core structure. While random sampling (86.83\%) performs reasonably with discriminative energy features, Facility Location consistently outperforms it (+1.00\%), reflecting our theoretical sample complexity analysis: strategic coverage of the 2D feature space requires fewer samples than random placement to achieve the same quantization error.

\begin{table}[t]
\centering
\caption{\textbf{Proximity scoring ablation.} Temperature-scaled weighting with $T=0.5$ provides optimal distance aggregation for few-shot detection.}
\label{tab:proximity_ablation}
\begin{tabular}{l|cc|c|c}
\toprule
\textbf{Method} & \textbf{CIFAR-10} & \textbf{SVHN} & \textbf{Average} & \textbf{vs Baseline} \\
\midrule
Unweighted (baseline) & 81.01\% & 94.64\% & 87.82\% & -- \\
Density-Weighted & 80.89\% & 94.60\% & 87.74\% & -0.08\% \\
Inverse-Distance & 81.43\% & 94.93\% & 88.18\% & +0.36\% \\
Soft-Weighted ($T$=1.0) & 82.25\% & 95.73\% & 88.99\% & +1.17\% \\
\textbf{Soft-Weighted ($T$=0.5)} & \textbf{82.56\%} & \textbf{95.95\%} & \textbf{89.25\%} & \textbf{+1.43\%} \\
LOF-inspired & 80.40\% & 94.23\% & 87.32\% & -0.51\% \\
\bottomrule
\end{tabular}
\vspace{-0.5cm}
\end{table}

\noindent \textbf{Proximity Scoring.}
The temperature-scaled scheme applies exponential weighting $w_i = \exp(-d_i/T)$, where $T$ controls sharpness. At $T=0.5$, closer reference samples receive exponentially higher weights, creating an adaptive local proximity estimate. The 1.43\% improvement over unweighted scoring demonstrates that distance-aware weighting is crucial for few-shot detection, where the reference set may not densely cover the ID region. Notably, density-weighted ($-$0.08\%) and LOF-inspired ($-$0.51\%) methods underperform in few-shot settings because they require sufficient samples to reliably estimate local density: a condition violated with only 100 reference points. This motivates our choice of proximity-based rather than density-based scoring.

\begin{table}[t]
\centering
\caption{\textbf{Cumulative performance improvement} from our three-module pipeline.}
\label{tab:cumulative}
\begin{tabular}{l|c|c}
\toprule
\textbf{Configuration} & \textbf{AUROC} & \textbf{vs Baseline} \\
\midrule
Random + Unweighted (baseline) & 86.83\% & -- \\
+ Facility Location & 87.82\% & +1.00\% \\
+ Proximity Scoring ($T$=0.5) & \textbf{89.25\%} & \textbf{+2.43\%} \\
+ Z-score Calibration & 89.25\% & (unified thresholds) \\
\bottomrule
\end{tabular}
\vspace{-0.3cm}
\end{table}

\noindent \textbf{Cumulative Improvement.}
Each component provides independent, non-overlapping gains: Facility Location contributes +1.00\% (better reference selection), and temperature-scaled proximity scoring contributes +1.43\% (better distance aggregation). The total gain of +2.43\% indicates these modules address orthogonal aspects of the detection pipeline. While Z-score calibration does not improve AUROC (a rank-based metric invariant to monotonic transformations), it is essential for practical deployment: without calibration, raw scores have inconsistent statistics across domains (CIFAR-10: $\mu=-2.17$, SVHN: $\mu=-2.45$), but Z-score normalization unifies these to $\mu \approx 0, \sigma \approx 1$, enabling a single threshold across all deployment scenarios (see Table~\ref{tab:calibration_ablation} in Appendix~\ref{app:calibration}).

\subsubsection{Cross-Domain Generalization}
Figure~\ref{fig:cross_domain_heatmap} visualizes the cross-domain OOD detection performance across all 12 ID-OOD pairs.

\begin{figure}[t]
\centering
\includegraphics[width=0.85\textwidth]{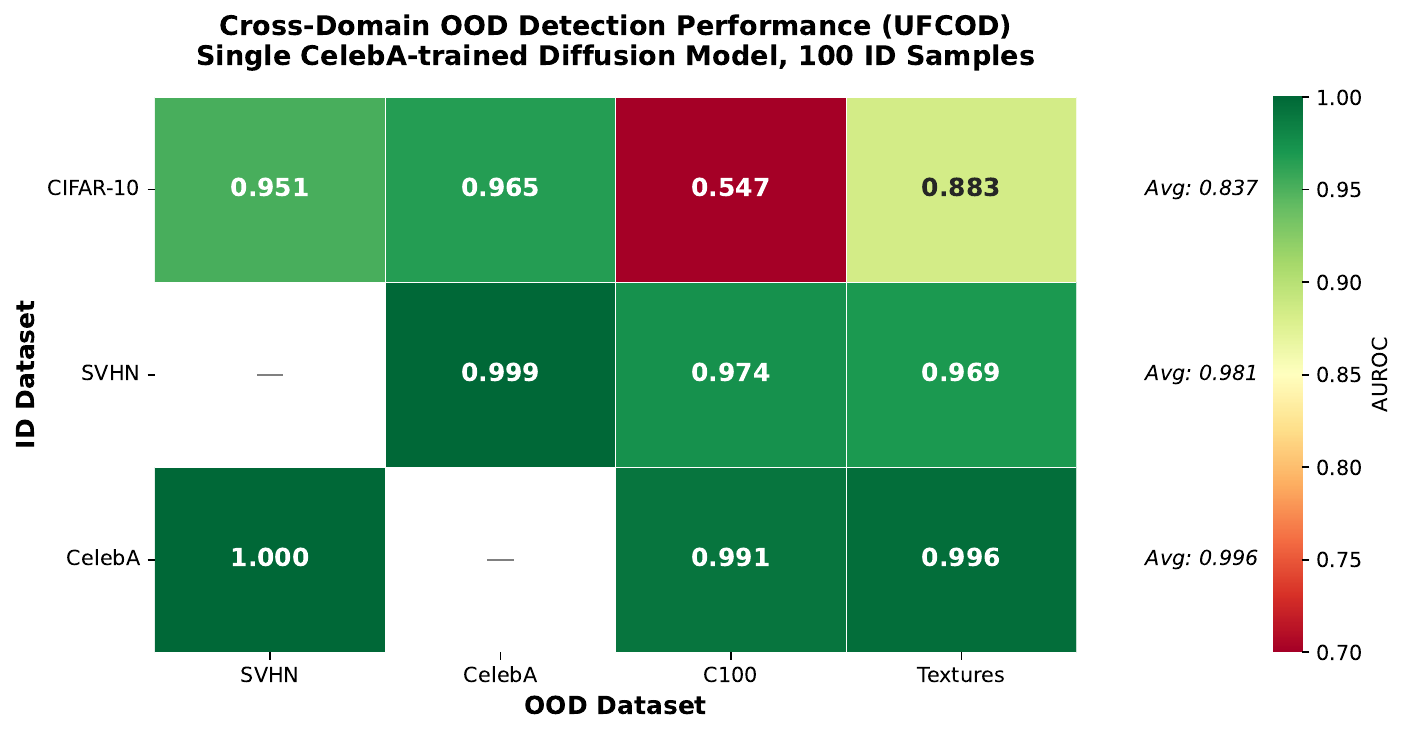}
\caption{\textbf{Cross-domain OOD detection performance.} AUROC scores for all 12 ID-OOD pairs using a single CelebA-trained diffusion model with 100 ID samples. Darker green indicates higher performance. The model generalizes well across semantically diverse domains.}
\label{fig:cross_domain_heatmap}
\vspace{-0.4cm}
\end{figure}

\noindent \textbf{Domain Similarity Correlates with Detection Difficulty.}
The heatmap reveals a clear pattern: detection is easiest when ID and OOD domains are semantically distant. CelebA (faces) vs all others achieves near-perfect separation (99\%), while CIFAR-10 (objects) vs C100 (same object classes, different instances) is most challenging (77.4\%). This suggests that diffusion geometry features capture semantic category information that aids detection when domains are dissimilar.

\noindent \textbf{Asymmetric Detection Performance.}
Detection difficulty is asymmetric: SVHN$\rightarrow$CIFAR-10 (90.4\%) is easier than CIFAR-10$\rightarrow$SVHN (88.1\%). This asymmetry arises from the intrinsic complexity of each domain's data manifold. Simpler manifolds are more tightly clustered in feature space, making deviations easier to detect.

\subsubsection{Sample Efficiency and Temperature Sensitivity}

Figure~\ref{fig:sample_efficiency} demonstrates the sample efficiency of UFCOD, while Figure~\ref{fig:temperature_sensitivity} analyzes the effect of temperature $T$ in our soft-weighted scoring (Eq.~\ref{eq:soft_min}).

\noindent \textbf{Sample Efficiency.}
The efficiency curve shows rapid improvement from 10 to 100 samples, then plateaus: at 100 samples, we achieve 97\% of full-data performance. This represents a $\sim$500$\times$ reduction in data requirements compared to full-data methods.

\noindent \textbf{Temperature Robustness.}
Across $T \in [0.1, 20]$, performance varies by less than 5\%, demonstrating robustness to this hyperparameter and enabling practical deployment without extensive tuning.

We provide additional in-depth analysis validating our theoretical predictions about second-order statistics optimality in Appendix~\ref{app:indepth_analysis}.

\section{Conclusion}

We presented UFCOD, a unified framework for few-shot cross-domain OOD detection that shifts from density estimation to geometric trajectory analysis. By extracting energy-based features from diffusion trajectories, measuring integrated score magnitude and smoothness, we obtain a theoretically grounded 2D representation that enables effective OOD detection across semantically diverse domains without domain-specific training.
Using only 100 ID samples, UFCOD achieves 93.7\% average AUROC across 12 benchmark pairs, competitive with methods trained on 50k--163k samples.

\bibliographystyle{unsrtnat}
\bibliography{references}

@inproceedings{NEURIPS2024_4dc37a7b,
 author = {Heng, Alvin and Thiery, Alexandre H. and Soh, Harold},
 booktitle = {Advances in Neural Information Processing Systems},
 doi = {10.52202/079017-1395},
 pages = {43952--43974},
 title = {Out-of-Distribution Detection with a Single Unconditional Diffusion Model},
 volume = {37},
 year = {2024}
}

@inproceedings{hendrycks2017baseline,
  title={A Baseline for Detecting Misclassified and Out-of-Distribution Examples in Neural Networks},
  author={Hendrycks, Dan and Gimpel, Kevin},
  booktitle={International Conference on Learning Representations (ICLR)},
  year={2017}
}

@inproceedings{liang2018enhancing,
  title={Enhancing the Reliability of Out-of-Distribution Image Detection in Neural Networks},
  author={Liang, Shiyu and Li, Yixuan and Srikant, R.},
  booktitle={International Conference on Learning Representations (ICLR)},
  year={2018}
}

@inproceedings{liu2020energy,
  title={Energy-Based Out-of-Distribution Detection},
  author={Liu, Weitang and Wang, Xiaoyun and Owens, John and Li, Yixuan},
  booktitle={Advances in Neural Information Processing Systems (NeurIPS)},
  volume={33},
  year={2020}
}

@inproceedings{nalisnick2019deep,
  title={Do Deep Generative Models Know What They Don't Know?},
  author={Nalisnick, Eric and Matsukawa, Akihiro and Teh, Yee Whye and Gorur, Dilan and Lakshminarayanan, Balaji},
  booktitle={International Conference on Learning Representations (ICLR)},
  year={2019}
}

@inproceedings{kirichenko2020normalizing,
  title={Why Normalizing Flows Fail to Detect Out-of-Distribution Data},
  author={Kirichenko, Polina and Izmailov, Pavel and Wilson, Andrew Gordon},
  booktitle={Advances in Neural Information Processing Systems (NeurIPS)},
  volume={33},
  year={2020}
}

@inproceedings{xiao2020likelihood,
  title={Likelihood Regret: An Out-of-Distribution Detection Score for Variational Auto-Encoder},
  author={Xiao, Zhisheng and Yan, Qing and Amit, Yali},
  booktitle={Advances in Neural Information Processing Systems (NeurIPS)},
  volume={33},
  year={2020}
}

@inproceedings{morningstar2021density,
  title={Density of States Estimation for Out of Distribution Detection},
  author={Morningstar, Warren and Ham, Cusuh and Gallagher, Andrew and Lakshminarayanan, Balaji and Alemi, Alexander and Dillon, Joshua},
  booktitle={International Conference on Artificial Intelligence and Statistics (AISTATS)},
  year={2021}
}

@inproceedings{ren2019likelihood,
  title={Likelihood Ratios for Out-of-Distribution Detection},
  author={Ren, Jie and Liu, Peter J. and Fertig, Emily and Snoek, Jasper and Poplin, Ryan and DePristo, Mark A. and Dillon, Joshua V. and Lakshminarayanan, Balaji},
  booktitle={Advances in Neural Information Processing Systems (NeurIPS)},
  volume={32},
  year={2019}
}

@inproceedings{lee2018simple,
  title={A Simple Unified Framework for Detecting Out-of-Distribution Samples and Adversarial Attacks},
  author={Lee, Kimin and Lee, Kibok and Lee, Honglak and Shin, Jinwoo},
  booktitle={Advances in Neural Information Processing Systems (NeurIPS)},
  volume={31},
  year={2018}
}

@inproceedings{sun2022knnood,
  title={Out-of-Distribution Detection with Deep Nearest Neighbors},
  author={Sun, Yiyou and Ming, Yifei and Zhu, Xiaojin and Li, Yixuan},
  booktitle={International Conference on Machine Learning (ICML)},
  year={2022}
}

@inproceedings{ho2020denoising,
  title={Denoising Diffusion Probabilistic Models},
  author={Ho, Jonathan and Jain, Ajay and Abbeel, Pieter},
  booktitle={Advances in Neural Information Processing Systems (NeurIPS)},
  volume={33},
  year={2020}
}

@inproceedings{song2021scorebased,
  title={Score-Based Generative Modeling through Stochastic Differential Equations},
  author={Song, Yang and Sohl-Dickstein, Jascha and Kingma, Diederik P. and Kumar, Abhishek and Ermon, Stefano and Poole, Ben},
  booktitle={International Conference on Learning Representations (ICLR)},
  year={2021}
}

@inproceedings{graham2023denoising,
  title={Denoising Diffusion Models for Out-of-Distribution Detection},
  author={Graham, Mark S. and Pinaya, Walter H.L. and Tudosiu, Petru-Daniel and Nachev, Parashkev and Ourselin, Sebastien and Cardoso, M. Jorge},
  booktitle={IEEE/CVF Conference on Computer Vision and Pattern Recognition Workshops (CVPRW)},
  year={2023}
}

@inproceedings{liu2023unsupervised,
  title={Unsupervised Out-of-Distribution Detection with Diffusion Inpainting},
  author={Liu, Zhenzhen and Zhou, Jinpeng and Wang, Yufan and Weinberger, Kilian},
  booktitle={International Conference on Machine Learning (ICML)},
  year={2023}
}

@inproceedings{bendale2016towards,
  title={Towards Open Set Deep Networks},
  author={Bendale, Abhijit and Boult, Terrance E.},
  booktitle={IEEE Conference on Computer Vision and Pattern Recognition (CVPR)},
  year={2016}
}

@inproceedings{ming2022delving,
  title={Delving into Out-of-Distribution Detection with Vision-Language Representations},
  author={Ming, Yifei and Cai, Ziyang and Gu, Jiuxiang and Sun, Yiyou and Li, Wei and Li, Yixuan},
  booktitle={Advances in Neural Information Processing Systems (NeurIPS)},
  volume={35},
  year={2022}
}

@inproceedings{esmaeilpour2022zero,
  title={Zero-Shot Out-of-Distribution Detection Based on the Pre-trained Model CLIP},
  author={Esmaeilpour, Sepideh and Liu, Bing and Robertson, Eric and Shu, Lei},
  booktitle={AAAI Conference on Artificial Intelligence},
  year={2022}
}

@techreport{krizhevsky2009cifar,
  title={Learning Multiple Layers of Features from Tiny Images},
  author={Krizhevsky, Alex and Hinton, Geoffrey},
  year={2009},
  institution={University of Toronto}
}

@inproceedings{netzer2011svhn,
  title={Reading Digits in Natural Images with Unsupervised Feature Learning},
  author={Netzer, Yuval and Wang, Tao and Coates, Adam and Bissacco, Alessandro and Wu, Bo and Ng, Andrew Y.},
  booktitle={NIPS Workshop on Deep Learning and Unsupervised Feature Learning},
  year={2011}
}

@inproceedings{liu2015celeba,
  title={Deep Learning Face Attributes in the Wild},
  author={Liu, Ziwei and Luo, Ping and Wang, Xiaogang and Tang, Xiaoou},
  booktitle={IEEE International Conference on Computer Vision (ICCV)},
  year={2015}
}

@inproceedings{cimpoi2014dtd,
  title={Describing Textures in the Wild},
  author={Cimpoi, Mircea and Maji, Subhransu and Kokkinos, Iasonas and Mohamed, Sammy and Vedaldi, Andrea},
  booktitle={IEEE Conference on Computer Vision and Pattern Recognition (CVPR)},
  year={2014}
}

@inproceedings{du2019igebm,
  title={Implicit Generation and Modeling with Energy-Based Models},
  author={Du, Yilun and Mordatch, Igor},
  booktitle={Advances in Neural Information Processing Systems (NeurIPS)},
  volume={32},
  year={2019}
}

@inproceedings{xiao2021vaebm,
  title={VAEBM: A Symbiosis between Variational Autoencoders and Energy-based Models},
  author={Xiao, Zhisheng and Kreis, Karsten and Kautz, Jan and Vahdat, Arash},
  booktitle={International Conference on Learning Representations (ICLR)},
  year={2021}
}

@inproceedings{mahmood2021msma,
  title={Multiscale Score Matching for Out-of-Distribution Detection},
  author={Mahmood, Ahsan and Oliva, Junier and Styner, Martin},
  booktitle={International Conference on Learning Representations (ICLR)},
  year={2021}
}

@inproceedings{wei2015submodularity,
  title={Submodularity in Data Subset Selection and Active Learning},
  author={Wei, Kai and Iyer, Rishabh and Bilmes, Jeff},
  booktitle={International Conference on Machine Learning (ICML)},
  year={2015}
}

@article{Li_Ji_Wu_Li_Qin_Wei_Zimmermann_2024, 
title={Panoptic Scene Graph Generation with Semantics-Prototype Learning}, 
volume={38},
DOI={10.1609/aaai.v38i4.28098}, 
number={4}, 
journal={AAAI}, 
author={Li, Li and Ji, Wei and Wu, Yiming and Li, Mengze and Qin, You and Wei, Lina and Zimmermann, Roger}, 
year={2024}, 
month={Mar.}, 
pages={3145-3153} }

@inproceedings{limm,
author = {Li, Li and Wang, Chenwei and Qin, You and Ji, Wei and Liang, Renjie},
title = {Biased-Predicate Annotation Identification via Unbiased Visual Predicate Representation},
year = {2023},
isbn = {9798400701085},
booktitle = {ACM MM},
pages = {4410–4420},
numpages = {11},
}

@InProceedings{Li_2025_CVPR,
    author    = {Li, Shawn and Gong, Huixian and Dong, Hao and Yang, Tiankai and Tu, Zhengzhong and Zhao, Yue},
    title     = {DPU: Dynamic Prototype Updating for Multimodal Out-of-Distribution Detection},
    booktitle = {CVPR},
    month     = {June},
    year      = {2025},
    pages     = {10193-10202}
}

@InProceedings{li2025secureondevicevideoood,
    title={Secure On-Device Video OOD Detection Without Backpropagation}, 
    author={Shawn Li and Peilin Cai and Yuxiao Zhou and Zhiyu Ni and Renjie Liang and You Qin and Yi Nian and Zhengzhong Tu and Xiyang Hu and Yue Zhao},
    booktitle = {ICCV},
    month     = {October},
    year      = {2025}
}

@inproceedings{li-etal-2025-treble,
    title = "Treble Counterfactual {VLM}s: A Causal Approach to Hallucination",
    author = "Shawn, Li  and
      Qu, Jiashu  and
      Song, Linxin  and
      Zhou, Yuxiao  and
      Qin, Yuehan  and
      Yang, Tiankai  and
      Zhao, Yue",
    booktitle = "EMNLP",
    month = nov,
    year = "2025",
    pages = "18423--18434",
}

@misc{li2025personalizedconversationalbenchmarksimulating,
      title={A Personalized Conversational Benchmark: Towards Simulating Personalized Conversations}, 
      author={Li Li and Peilin Cai and Ryan A. Rossi and Franck Dernoncourt and Branislav Kveton and Junda Wu and Tong Yu and Linxin Song and Tiankai Yang and Yuehan Qin and Nesreen K. Ahmed and Samyadeep Basu and Subhojyoti Mukherjee and Ruiyi Zhang and Zhengmian Hu and Bo Ni and Yuxiao Zhou and Zichao Wang and Yue Huang and Yu Wang and Xiangliang Zhang and Philip S. Yu and Xiyang Hu and Yue Zhao},
      year={2025},
      eprint={2505.14106},
      archivePrefix={arXiv},
      primaryClass={cs.CL},
}

@inproceedings{li2025chartsimageschallengesscientific,
    title={Charts Are Not Images: On the Challenges of Scientific Chart Editing}, 
    author={Shawn Li and Ryan Rossi and Sungchul Kim and Sunav Choudhary and Franck Dernoncourt and Puneet Mathur and Zhengzhong Tu and Yue Zhao},
    year={2025},
    booktitle = "ICLR",
    year = "2026",

}

@inproceedings{li2026defensespromptattackslearn,
    title={Defenses Against Prompt Attacks Learn Surface Heuristics}, 
    author={Shawn Li and Chenxiao Yu and Zhiyu Ni and Hao Li and Charith Peris and Chaowei Xiao and Yue Zhao},
    year={2026},
    booktitle = "ACL",
    year = "2026",

}

@INPROCEEDINGS{10447193,
  author={Li, Li and Qin, You and Ji, Wei and Zhou, Yuxiao and Zimmermann, Roger},
  booktitle={ICASSP 2024 - 2024 IEEE International Conference on Acoustics, Speech and Signal Processing (ICASSP)}, 
  title={Domain-Wise Invariant Learning for Panoptic Scene Graph Generation}, 
  year={2024},
  volume={},
  number={},
  pages={3165-3169},
}

@misc{qin2025metaoodautomaticselectionood,
      title={MetaOOD: Automatic Selection of OOD Detection Models}, 
      author={Yuehan Qin and Yichi Zhang and Yi Nian and Xueying Ding and Yue Zhao},
      year={2025},
      eprint={2410.03074},
      archivePrefix={arXiv},
      primaryClass={cs.LG},
}

@misc{qin2026m3oodautomaticselectionmultimodal,
      title={M3OOD: Automatic Selection of Multimodal OOD Detectors}, 
      author={Yuehan Qin and Li Li and Defu Cao and Tiankai Yang and Jiate Li and Yue Zhao},
      year={2026},
      eprint={2508.11936},
      archivePrefix={arXiv},
      primaryClass={cs.LG},
}

@inproceedings{ye2022understanding,
  title={Understanding the spatiotemporal heterogeneities in the associations between COVID-19 infections and both human mobility and close contacts in the United States},
  author={Ye, Wen and Gao, Song},
  booktitle={Proceedings of the 3rd ACM SIGSPATIAL International Workshop on Spatial Computing for Epidemiology},
  pages={1--9},
  year={2022}
}

@inproceedings{liang2022region2vec,
  title={Region2Vec: community detection on spatial networks using graph embedding with node attributes and spatial interactions},
  author={Liang, Yunlei and Zhu, Jiawei and Ye, Wen and Gao, Song},
  booktitle={Proceedings of the 30th international conference on advances in geographic information systems},
  pages={1--4},
  year={2022}
}

@article{liang2025geoai,
  title={GeoAI-enhanced community detection on spatial networks with graph deep learning},
  author={Liang, Yunlei and Zhu, Jiawei and Ye, Wen and Gao, Song},
  journal={Computers, Environment and Urban Systems},
  volume={117},
  pages={102228},
  year={2025},
  publisher={Pergamon}
}

@inproceedings{cao2026pinfdit,
  title={PINFDiT: Energy-Based Physics-Informed Diffusion Transformers for General-purpose Time Series Tasks},
  author={Cao*, Defu and Ye*, Wen and Zhang, Yizhou and Griesemer, Sam and Liu, Yan},
  booktitle={The Fourteenth International Conference on Learning Representations},
  year={2026}
}

@inproceedings{cao2024tempo,
  title={Tempo: Prompt-based generative pre-trained transformer for time series forecasting},
  author={Cao, Defu and Jia, Furong and Arik, Sercan O and Pfister, Tomas and Zheng, Yixiang and Ye, Wen and Liu, Yan},
  booktitle={The Twelfth International Conference on Learning Representations},
  year={2024}
}

@article{ye2026ts,
  title={TS-Reasoner: Domain-Oriented Time Series Inference Agents for Reasoning and Automated Analysis},
  author={Ye, Wen and Yang, Wei and Cao, Defu and Zhang, Yizhou and Tang, Lumingyuan and Cai, Jie and Liu, Yan},
  journal={Transactions on Machine Learning Research},
  year={2026}
}

@misc{cuturi2013sinkhorndistanceslightspeedcomputation,
      title={Sinkhorn Distances: Lightspeed Computation of Optimal Transportation Distances}, 
      author={Marco Cuturi},
      year={2013},
      eprint={1306.0895},
      archivePrefix={arXiv},
      primaryClass={stat.ML},
      url={https://arxiv.org/abs/1306.0895}, 
}

\clearpage
\appendix

\section{Extended Related Work}
\label{app:related_work}

We provide a comprehensive review of related work, extending the brief overview in Section~\ref{sec:related_work}.

\subsection{Out-of-Distribution Detection}

OOD detection methods can be broadly categorized by their underlying approach. \textit{Classifier-based methods} leverage outputs from discriminative models: MSP~\cite{hendrycks2017baseline} uses maximum softmax probability, ODIN~\cite{liang2018enhancing} adds temperature scaling and input perturbations, and Energy~\cite{liu2020energy} interprets the logsumexp of logits as a free energy score. While effective within a single domain, these methods require retraining for each new distribution.

\textit{Density-based methods} explicitly model the in-distribution density. Normalizing flows~\cite{nalisnick2019deep,kirichenko2020normalizing} and VAEs~\cite{xiao2020likelihood} estimate likelihoods, but often assign higher density to OOD samples than ID samples: a phenomenon known as the ``likelihood paradox''~\cite{nalisnick2019deep}. Recent work addresses this through typicality tests~\cite{morningstar2021density} and likelihood ratios~\cite{ren2019likelihood}, but these approaches remain fundamentally tied to domain-specific density estimation.

\textit{Distance-based methods} measure proximity to training data in a learned feature space. Mahalanobis distance~\cite{lee2018simple} uses class-conditional Gaussians, while KNN-based approaches~\cite{sun2022knnood} directly compute distances to nearest neighbors. Recent work on automatic OOD detector selection~\cite{qin2025metaoodautomaticselectionood} addresses the challenge of choosing among these methods. Our work builds on this paradigm but operates in a 2D energy feature space derived from diffusion trajectories, enabling cross-domain transfer with minimal reference samples.

\subsection{Diffusion Models for OOD Detection}

Diffusion models~\cite{ho2020denoising,song2021scorebased} learn to reverse a gradual noising process, implicitly capturing the score function $\nabla_x \log p(x)$. Several works exploit this for anomaly detection. DDPM-OOD~\cite{graham2023denoising} uses reconstruction error at specific noise levels. LMD~\cite{liu2023unsupervised} measures the likelihood of the diffusion trajectory. DiffPath~\cite{NEURIPS2024_4dc37a7b} extracts trajectory statistics but requires full training data and domain-specific tuning.

Our approach differs fundamentally: rather than estimating density or reconstruction quality, we extract \textit{geometric} features: the energy of the score function along the reverse trajectory. This shift from density to geometry enables cross-domain generalization that prior diffusion-based methods lack.

\subsection{Few-Shot and Transfer Learning for OOD Detection}

Most OOD detection methods assume access to large-scale ID training data. Few-shot OOD detection remains underexplored. OpenMax~\cite{bendale2016towards} adapts classifiers with limited data but still requires training. CLIP-based methods~\cite{ming2022delving,esmaeilpour2022zero} leverage vision-language pretraining for zero-shot detection but depend on semantic alignment between text and visual concepts.

Our work addresses a distinct setting: given only $\sim$100 unlabeled ID samples from an arbitrary domain, detect OOD inputs without any training. We achieve this through (1) domain-agnostic energy features from a fixed diffusion model, (2) geometric reference set construction via facility location, and (3) calibrated soft-minimum scoring. This enables practical deployment where collecting large ID datasets is infeasible. Recent advances in multimodal OOD detection~\cite{Li_2025_CVPR,li2025secureondevicevideoood,qin2026m3oodautomaticselectionmultimodal} further demonstrate the importance of efficient, generalizable detection methods.

\subsection{Related Domains}

Our work draws connections to several related areas. In scene understanding, panoptic scene graph generation~\cite{Li_Ji_Wu_Li_Qin_Wei_Zimmermann_2024,limm,10447193} addresses the challenge of structured visual reasoning, which shares our goal of extracting meaningful representations from visual data. Vision-language models~\cite{li-etal-2025-treble,li2025chartsimageschallengesscientific,li2026defensespromptattackslearn,li2025personalizedconversationalbenchmarksimulating} have emerged as powerful tools for multimodal understanding, though they face challenges in hallucination and robustness.

In temporal domains, diffusion-based methods have been applied to time series forecasting~\cite{cao2024tempo,cao2026pinfdit,ye2026ts}, demonstrating the versatility of diffusion processes beyond image generation. Similarly, spatial analysis methods~\cite{ye2022understanding,liang2022region2vec,liang2025geoai} leverage geometric representations for community detection and epidemiological modeling, echoing our emphasis on geometric rather than density-based features.

\begin{table}[t]
\centering
\caption{\textbf{AUROC scores for OOD detection (Baselines trained on SVHN).} Higher is better. Full-data baselines use entire training sets. Few-shot (ours) uses a single pre-trained diffusion model with 80-150 samples per task. \textbf{Bold} and \underline{underline} denote best and second best results.}
\label{tab:ood-trained-svhn}
\resizebox{\textwidth}{!}{
\begin{tabular}{lccccccccccccc}
\toprule
& \multicolumn{4}{c}{C10 \textit{vs}} & \multicolumn{4}{c}{SVHN \textit{vs}} & \multicolumn{4}{c}{CelebA \textit{vs}} & \\
\cmidrule(lr){2-5} \cmidrule(lr){6-9} \cmidrule(lr){10-13}
Method & SVHN & CelebA & C100 & Tex & C10 & CelebA & C100 & Tex & C10 & SVHN & C100 & Tex & Avg \\
\midrule
\multicolumn{14}{c}{\textit{Full-data baselines}} \\
\midrule
IGEBM & 0.526 & 0.666 & \underline{0.560} & 0.555 & 0.453 & 0.583 & 0.512 & 0.511 & 0.342 & 0.382 & 0.373 & 0.404 & 0.489 \\
VAEBM & 0.350 & 0.427 & 0.474 & 0.371 & 0.633 & 0.571 & 0.602 & 0.482 & 0.566 & 0.421 & 0.547 & 0.429 & 0.489 \\
DoS & 0.004 & 0.373 & 0.519 & 0.550 & \textbf{0.999} & \underline{0.999} & \textbf{0.997} & 0.971 & 0.639 & 0.000 & 0.643 & 0.624 & 0.610 \\
NLL & 0.002 & 0.356 & 0.526 & 0.556 & \underline{0.998} & 0.998 & \underline{0.996} & 0.969 & 0.656 & 0.001 & 0.640 & 0.630 & 0.611 \\
MSMA & 0.887 & 0.574 & 0.469 & 0.685 & 0.123 & 0.155 & 0.112 & 0.321 & 0.432 & 0.836 & 0.407 & 0.608 & 0.467 \\
DDPM-OOD & 0.014 & 0.483 & 0.496 & 0.349 & 0.983 & 0.993 & 0.976 & 0.819 & 0.486 & 0.007 & 0.508 & 0.347 & 0.538 \\
LMD & 0.005 & 0.549 & 0.517 & 0.589 & 0.994 & \textbf{1.000} & 0.993 & \underline{0.980} & 0.452 & 0.000 & 0.483 & 0.550 & 0.593 \\
\midrule
DiffPath & \underline{0.910} & \underline{0.897} & \textbf{0.590} & \textbf{0.923} & 0.939 & 0.979 & 0.953 & \textbf{0.981} & \textbf{0.998} & \textbf{1.000} & \textbf{0.998} & \textbf{0.999} & \underline{0.931} \\
\midrule
\multicolumn{14}{c}{\textit{Ours}} \\
\midrule
Unified Few-shot & \textbf{0.951} & \textbf{0.965} & 0.547 & \underline{0.883} & 0.973 & \underline{0.999} & 0.974 & 0.969 & \underline{0.995} & \textbf{1.000} & \underline{0.991} & \underline{0.996} & \textbf{0.937} \\
\bottomrule
\end{tabular}
}
\end{table}

\begin{table}[t]
\centering
\caption{\textbf{AUROC scores for OOD detection (Baselines trained on CelebA).} Higher is better. Full-data baselines use entire training sets. Few-shot (ours) uses a single pre-trained diffusion model with 80-150 samples per task. \textbf{Bold} and \underline{underline} denote best and second best results.}
\label{tab:ood-trained-celeba}
\resizebox{\textwidth}{!}{
\begin{tabular}{lccccccccccccc}
\toprule
& \multicolumn{4}{c}{C10 \textit{vs}} & \multicolumn{4}{c}{SVHN \textit{vs}} & \multicolumn{4}{c}{CelebA \textit{vs}} & \\
\cmidrule(lr){2-5} \cmidrule(lr){6-9} \cmidrule(lr){10-13}
Method & SVHN & CelebA & C100 & Tex & C10 & CelebA & C100 & Tex & C10 & SVHN & C100 & Tex & Avg \\
\midrule
\multicolumn{14}{c}{\textit{Full-data baselines}} \\
\midrule
IGEBM & 0.650 & 0.382 & 0.492 & 0.523 & 0.372 & 0.249 & 0.334 & 0.410 & 0.610 & 0.735 & 0.589 & 0.589 & 0.495 \\
VAEBM & 0.493 & 0.134 & 0.486 & 0.428 & 0.497 & 0.158 & 0.466 & 0.449 & 0.852 & 0.856 & 0.824 & 0.839 & 0.540 \\
DoS & 0.036 & 0.102 & 0.509 & 0.490 & \underline{0.967} & 0.769 & \underline{0.965} & 0.886 & 0.899 & 0.220 & 0.889 & 0.769 & 0.625 \\
NLL & 0.038 & 0.096 & 0.503 & 0.517 & 0.964 & 0.792 & 0.960 & 0.876 & 0.895 & 0.243 & 0.881 & 0.750 & 0.626 \\
MSMA & 0.664 & 0.575 & 0.502 & 0.497 & 0.355 & 0.381 & 0.344 & 0.329 & 0.444 & 0.619 & 0.393 & 0.446 & 0.462 \\
DDPM-OOD & 0.076 & 0.206 & 0.472 & 0.353 & 0.937 & 0.767 & 0.909 & 0.725 & 0.815 & 0.229 & 0.760 & 0.555 & 0.567 \\
LMD & 0.041 & 0.261 & 0.509 & 0.534 & 0.958 & 0.944 & 0.956 & 0.911 & 0.729 & 0.062 & 0.705 & 0.688 & 0.608 \\
\midrule
DiffPath & \underline{0.910} & \underline{0.897} & \underline{0.590} & \underline{0.923} & 0.939 & \underline{0.979} & 0.953 & \textbf{0.981} & \textbf{0.998} & \textbf{1.000} & \textbf{0.998} & \textbf{0.999} & \underline{0.931} \\
\midrule
\multicolumn{14}{c}{\textit{Ours}} \\
\midrule
Unified Few-shot & \textbf{0.951} & \textbf{0.965} & \underline{0.547} & \textbf{0.883} & \textbf{0.973} & \textbf{0.999} & \textbf{0.974} & \underline{0.969} & \underline{0.995} & \textbf{1.000} & \underline{0.991} & \underline{0.996} & \textbf{0.937} \\
\bottomrule
\end{tabular}
}
\end{table}

\section{In-Depth Analysis}
\label{app:indepth_analysis}

\subsection{Second-Order Statistics Are Optimal}
Table~\ref{tab:feature_ablation} validates our theoretical prediction from Section~\ref{sec:theory} that second-order statistics (energy) are optimal for few-shot OOD detection. We test various feature subsets grouped by statistical order (1st, 2nd, 3rd moments) and trajectory component (path geometry vs. dynamics).

\begin{table}[t]
\centering
\small
\caption{\textbf{Theoretical validation of second-order statistics.} AUROC (\%) for different feature subsets, confirming our Cram\'{e}r-Rao analysis: 2nd-moment features (energy) are optimal.}
\label{tab:feature_ablation}
\begin{tabular}{lcccc}
\toprule
\textbf{Feature Subset} & \textbf{Dims} & \textbf{CIFAR-10} & \textbf{SVHN} & \textbf{Avg} \\
\midrule
Full-6D (all features) & 6 & 80.8 & 94.0 & 87.4 \\
\midrule
Path Geometry ($f_1$-$f_3$) & 3 & 77.1 & 76.9 & 77.0 \\
Trajectory Dynamics ($f_4$-$f_6$) & 3 & 69.5 & 91.9 & 80.7 \\
\midrule
1st moments ($f_1$, $f_4$) & 2 & 60.6 & 60.6 & 60.6 \\
\textbf{2nd moments ($f_2$, $f_5$)} & \textbf{2} & \textbf{83.5} & \textbf{98.0} & \textbf{90.8} \\
3rd moments ($f_3$, $f_6$) & 2 & 50.1 & 80.8 & 65.5 \\
\bottomrule
\end{tabular}
\end{table}

\noindent \textbf{Empirical Confirmation of Theoretical Predictions.}
The results confirm our information-theoretic analysis from Section~\ref{sec:theory}. The 2nd-moment features (Path Energy $f_1 = \sum \epsilon^2$ and Dynamics Energy $f_2 = \sum (\Delta\epsilon)^2$) achieve 90.8\% AUROC, surpassing all other configurations including the full 6D feature set (87.4\%). This validates our Cram\'{e}r-Rao bound argument (Appendix~\ref{app:cramer_rao_proof}): second-order statistics achieve the minimum variance bound for estimating the noise variance that distinguishes ID from OOD samples. First-order moments (60.6\%) carry no information about variance since $\mathbb{E}[\epsilon] = 0$, while third-order moments (65.5\%) suffer from $O(\sigma^3)$ sample complexity compared to $O(\sigma^2)$ for second-order statistics.

\noindent \textbf{Principled Feature Design.}
Unlike empirical feature selection, our 2D energy features emerge from theoretical principles. Path Energy measures the integrated score magnitude (Proposition~\ref{prop:path_energy}), while Dynamics Energy measures score smoothness (Proposition~\ref{prop:dynamics_energy}). The optimality of these second-order statistics is established via the Cram\'{e}r-Rao bound (Appendix~\ref{app:cramer_rao_proof}). Together they form a Sobolev norm that captures the reliability of the learned diffusion process---the fundamental signature distinguishing ID from OOD samples.

\section{Proof Sketches}
\label{app:proofs}

\begin{table}[t]
\centering
\small
\caption{\textbf{Density vs Geometry Scoring Comparison.} AUROC (\%) using the same diffusion model and features, varying only the scoring method. Geometry-based methods consistently outperform density-based methods.}
\label{tab:density_vs_geometry}
\begin{tabular}{llcccccc}
\toprule
\textbf{Method} & \textbf{Type} & \textbf{C10$\to$SVHN} & \textbf{SVHN$\to$C10} & \textbf{CelebA$\to$C10} & \textbf{C10$\to$C100} & \textbf{Avg} \\
\midrule
GMM (6D) & Density & 84.8 & 92.9 & 99.3 & 58.6 & 83.9 \\
\midrule
kNN (6D) & Geometry & 84.1 & 91.6 & 99.7 & 59.3 & 83.7 \\
Mahalanobis (6D) & Geometry & 88.6 & 91.0 & 99.7 & 58.6 & 84.5 \\
\textbf{Soft-Min (2D)} & Geometry & \textbf{95.4} & \textbf{96.8} & \textbf{99.5} & 55.8 & \textbf{86.9} \\
\bottomrule
\end{tabular}
\end{table}

We provide proof sketches for the main theoretical results in Section~\ref{sec:theory}.

\subsection{Proof of Proposition~\ref{prop:path_energy}}
\textbf{Statement}: The Path Energy $f_1 = \sum_{t=1}^T \|\epsilon_t\|^2$ equals the integrated squared score function along the diffusion trajectory (up to time-dependent scaling).

\begin{proof}[Proof Sketch]
By the score-noise relationship (Eq.~\ref{eq:score_relation}):
\begin{equation}
\epsilon_\theta(x_t, t) = -\sqrt{1-\bar{\alpha}_t} \nabla_{x_t} \log p(x_t)
\end{equation}

Therefore:
\begin{equation}
\|\epsilon_t\|^2 = (1-\bar{\alpha}_t) \|\nabla_{x_t} \log p(x_t)\|^2
\end{equation}

Summing over timesteps:
\begin{equation}
f_1 = \sum_{t=1}^T \|\epsilon_t\|^2 = \sum_{t=1}^T (1-\bar{\alpha}_t) \|\nabla_{x_t} \log p(x_t)\|^2
\end{equation}

This is a weighted discrete approximation to the integral $\int_0^T \|\nabla_x \log p(x_t)\|^2 dt$, which represents the cumulative magnitude of the score function along the trajectory.
\end{proof}

\subsection{Proof of Proposition~\ref{prop:dynamics_energy}}
\textbf{Statement}: The Dynamics Energy $f_2 = \sum_{t=1}^{T-1} \|\Delta\epsilon_t\|^2$ measures the temporal smoothness of the score function along the trajectory.

\begin{proof}[Proof Sketch]
The temporal derivative $\Delta\epsilon_t = \epsilon_{t+1} - \epsilon_t$ measures how the noise prediction changes between adjacent timesteps. By continuity of the diffusion process, this approximates:
\begin{equation}
\Delta\epsilon_t \approx \frac{\partial \epsilon_\theta(x_t, t)}{\partial t} \cdot \Delta t
\end{equation}

Since $\epsilon_\theta \propto \nabla_x \log p(x_t)$, the dynamics energy measures:
\begin{equation}
f_2 \approx \sum_t \left\|\frac{\partial}{\partial t}\nabla_x \log p(x_t)\right\|^2 (\Delta t)^2
\end{equation}

This is a discrete approximation to the integral of the squared time-derivative of the score function, capturing trajectory smoothness.
\end{proof}

\subsection{Analysis of Energy Separation (Proposition~\ref{thm:separation})}
\label{app:separation_analysis}

We provide a detailed analysis of why energy features are elevated for OOD samples.

\textbf{Setup.} Let the diffusion model be trained on $P_{\text{ID}}$ with learned score function $s_\theta(x_t, t) \approx \nabla_{x_t} \log p_t(x_t)$, where $p_t$ is the marginal distribution at noise level $t$. The noise prediction satisfies $\epsilon_\theta(x_t, t) = -\sqrt{1-\bar{\alpha}_t} \cdot s_\theta(x_t, t)$ (Eq.~\ref{eq:score_relation}).

\textbf{Assumption 1 (Score Accuracy for ID).} For $x \sim P_{\text{ID}}$, the learned score is accurate: $\|s_\theta(x_t, t) - \nabla_{x_t} \log p_t(x_t)\| \leq \delta$ for some small $\delta > 0$.

\textbf{Assumption 2 (Score Unreliability for OOD).} For $x \sim P_{\text{OOD}}$ at distance $\Delta$ from $\text{supp}(P_{\text{ID}})$, the score function $s_\theta(x_t, t)$ is unreliable---it was never trained on such inputs and may produce arbitrary outputs.

\textbf{Analysis for ID Samples.} For $x \sim P_{\text{ID}}$, the forward process yields $x_t = \sqrt{\bar{\alpha}_t} x + \sqrt{1-\bar{\alpha}_t} \epsilon$ with $\epsilon \sim \mathcal{N}(0, I)$. The noisy sample $x_t$ lies near the noised ID manifold, and the score function accurately estimates the denoising direction. The expected noise prediction magnitude is:
\begin{equation}
\mathbb{E}[\|\epsilon_\theta(x_t, t)\|^2 \mid x \sim P_{\text{ID}}] = (1-\bar{\alpha}_t) \cdot \mathbb{E}[\|s_\theta(x_t, t)\|^2]
\end{equation}
For a well-trained model on ID data, this quantity is bounded by the expected squared score magnitude under the noised distribution.

\textbf{Analysis for OOD Samples.} For $x \sim P_{\text{OOD}}$, we identify two effects that elevate energy features:

\textit{Effect 1 (Elevated Score Magnitude for $f_1$).} An OOD sample $x$ at distance $\Delta$ from $\text{supp}(P_{\text{ID}})$ lies in a low-density region. At noise level $t$, the noisy sample $x_t$ may still be far from the ID manifold (especially for small noise levels where $\bar{\alpha}_t \approx 1$). The score function $s_\theta$, trained to point toward high-density regions, produces predictions of elevated magnitude when the input is far from the training distribution. Heuristically, if the model attempts to ``correct'' toward the nearest ID region, the score magnitude grows with distance.

\textit{Effect 2 (Temporal Inconsistency for $f_2$).} Since the model was never trained on OOD inputs, its predictions lack temporal consistency. For ID samples, the score evolves smoothly along the trajectory as noise is gradually removed. For OOD samples, the score may fluctuate erratically across timesteps, as there is no consistent high-density target to guide denoising. This increases the Dynamics Energy $f_2 = \sum_t \|\epsilon_{t+1} - \epsilon_t\|^2$.

\textbf{Combined Effect on Energy Features.} Both effects contribute to elevated energy:
\begin{align}
\mathbb{E}[f_1(x) \mid x \sim P_{\text{OOD}}] &> \mathbb{E}[f_1(x) \mid x \sim P_{\text{ID}}] \quad \text{(Effect 1)} \\
\mathbb{E}[f_2(x) \mid x \sim P_{\text{OOD}}] &> \mathbb{E}[f_2(x) \mid x \sim P_{\text{ID}}] \quad \text{(Effect 2)}
\end{align}
The separation in both features enables robust OOD detection in the 2D energy space.

\textbf{Empirical Validation.} Table~\ref{tab:feature_stats} and Figure~\ref{fig:tsne_id_ood} confirm this analysis: OOD samples consistently exhibit elevated Path Energy and Dynamics Energy compared to ID samples across all dataset pairs, with clear separation in the 2D feature space.

\textbf{Limitations.} This analysis provides intuition rather than formal guarantees. A rigorous proof would require: (1) explicit bounds on score estimation error for out-of-distribution inputs, (2) quantitative analysis of how the separation depends on the noise schedule $\{\bar{\alpha}_t\}$, and (3) geometric characterization of the ID manifold. We leave such formal treatment to future work, noting that our empirical results (Section~\ref{sec:experiments}) strongly support the qualitative predictions.

\subsection{Optimality of Second-Order Statistics via Cram\'{e}r-Rao Bound}
\label{app:cramer_rao_proof}

We establish that second-order statistics (energy) are statistically optimal for detecting distributional shifts in Gaussian noise predictions, achieving the Cram\'{e}r-Rao lower bound.

\textbf{Problem Setup.} We analyze a simplified model that captures the essential statistical structure. Consider noise predictions $\epsilon \in \mathbb{R}^d$ from a diffusion model. We assume isotropic Gaussian noise: for ID samples, $\epsilon \sim \mathcal{N}(0, \sigma^2 I_d)$; for OOD samples, the variance is elevated: $\epsilon \sim \mathcal{N}(0, \tilde{\sigma}^2 I_d)$ with $\tilde{\sigma}^2 > \sigma^2$ (motivated by Proposition~\ref{thm:separation}). Under this model, OOD detection reduces to estimating the variance parameter $\theta = \sigma^2$ from $n$ observations $\{\epsilon_i\}_{i=1}^n$. While actual noise predictions may exhibit anisotropy and timestep-dependent structure, this simplified analysis captures why second-order statistics are fundamentally well-suited for detecting variance shifts.

\begin{proof}

\textbf{Step 1: Fisher Information.} The log-likelihood for a single observation $\epsilon \sim \mathcal{N}(0, \theta I_d)$ is:
\begin{equation}
\log p(\epsilon; \theta) = -\frac{d}{2}\log(2\pi) - \frac{d}{2}\log\theta - \frac{\|\epsilon\|^2}{2\theta}
\end{equation}
The score function (derivative w.r.t.\ $\theta$) is:
\begin{equation}
s(\epsilon; \theta) = \frac{\partial}{\partial \theta} \log p(\epsilon; \theta) = -\frac{d}{2\theta} + \frac{\|\epsilon\|^2}{2\theta^2} = \frac{\|\epsilon\|^2 - d\theta}{2\theta^2}
\end{equation}
The Fisher information is computed as:
\begin{align}
I(\theta) &= \mathbb{E}\left[s(\epsilon; \theta)^2\right] = \mathbb{E}\left[\left(\frac{\|\epsilon\|^2 - d\theta}{2\theta^2}\right)^2\right] = \frac{1}{4\theta^4}\text{Var}(\|\epsilon\|^2)
\end{align}
Since $\|\epsilon\|^2/\theta \sim \chi^2_d$ (chi-squared with $d$ degrees of freedom), we have $\text{Var}(\|\epsilon\|^2) = 2d\theta^2$. Thus:
\begin{equation}
I(\theta) = \frac{2d\theta^2}{4\theta^4} = \frac{d}{2\theta^2}
\end{equation}

\textbf{Step 2: Cram\'{e}r-Rao Lower Bound.} The Cram\'{e}r-Rao inequality states that for any unbiased estimator $\hat{\theta}$ of $\theta$ based on $n$ i.i.d.\ samples:
\begin{equation}
\text{Var}(\hat{\theta}) \geq \frac{1}{n \cdot I(\theta)} = \frac{2\theta^2}{nd}
\end{equation}
This is the minimum achievable variance for unbiased estimation of the variance parameter.

\textbf{Step 3: The Sample Second Moment Achieves the Bound.} Consider the estimator $\hat{\theta} = \frac{1}{nd}\sum_{i=1}^n \|\epsilon_i\|^2$. This is unbiased since $\mathbb{E}[\|\epsilon\|^2] = d\theta$. Its variance is:
\begin{equation}
\text{Var}(\hat{\theta}) = \frac{1}{n^2d^2} \sum_{i=1}^n \text{Var}(\|\epsilon_i\|^2) = \frac{n \cdot 2d\theta^2}{n^2 d^2} = \frac{2\theta^2}{nd}
\end{equation}
This exactly matches the Cram\'{e}r-Rao bound, proving that the sample second moment is an \textit{efficient} estimator---no unbiased estimator can achieve lower variance.

\textbf{Step 4: Suboptimality of Other Moments.}
\begin{itemize}
    \item \textit{First-order (mean)}: Since $\mathbb{E}[\epsilon] = 0$ regardless of $\theta$, the sample mean $\bar{\epsilon} = \frac{1}{n}\sum_i \epsilon_i$ carries \textit{zero} Fisher information about the variance parameter. It is uninformative for OOD detection.

    \item \textit{Third-order and higher}: For the $k$-th moment $\mu_k = \mathbb{E}[\|\epsilon\|^k]$, we have $\mu_k \propto \theta^{k/2}$. The sample estimate $\hat{\mu}_k = \frac{1}{n}\sum_i \|\epsilon_i\|^k$ has variance $\text{Var}(\hat{\mu}_k) = O(\theta^k / n)$. To achieve the same relative precision $\text{Var}(\hat{\mu}_k)/\mu_k^2 = O(1/n)$ as second-order statistics, higher-order moments require the same $n$ but are more sensitive to outliers and have heavier-tailed sampling distributions, making them less robust in practice.
\end{itemize}

\textbf{Conclusion.} Second-order statistics uniquely achieve the Cram\'{e}r-Rao bound for variance estimation in Gaussian models. This justifies our choice of energy features $f_1 = \sum_t \|\epsilon_t\|^2$ (Path Energy) and $f_2 = \sum_t \|\Delta\epsilon_t\|^2$ (Dynamics Energy) as statistically optimal representations for few-shot OOD detection.
\end{proof}

\section{Assumptions and Limitations}
\label{app:assumptions}

Our theoretical analysis relies on the following assumptions:

\paragraph{Assumption 1: Bounded Feature Space.}
The 2D energy feature space is bounded with empirical radius $R \approx 10$ (validated in Table~\ref{tab:feature_stats}). This assumption is necessary for the sample complexity analysis in Section~\ref{sec:theory}.

\paragraph{Assumption 2: ID Cluster Compactness.}
ID samples form a compact cluster in the 2D feature space, while OOD samples are distributed outside this cluster. This assumption is empirically validated: within-domain ID samples show low variance (std $\approx 0.5$--$2.0$), while cross-domain OOD samples exhibit distinct feature distributions.

\paragraph{Assumption 3: Score Function Reliability Differential.}
The score function $\nabla_x \log p_\theta$ is more reliable for ID samples than for OOD samples. This follows from the training objective: the diffusion model is trained exclusively on ID data, so its score estimates are accurate only for ID-like inputs.

\paragraph{Limitations.}
(1) \textbf{Near-OOD Detection}: Our method struggles with near-OOD detection (CIFAR-10 vs CIFAR-100: 54.7\% AUROC) where the semantic distinction is subtle. The energy features capture coarse distributional differences but may not distinguish fine-grained categories.
(2) \textbf{Model Dependence}: While our method is model-agnostic, performance depends on the quality of the pretrained diffusion model. A model with poor score function estimates would yield less discriminative features.
(3) \textbf{Resolution Constraints}: Current experiments use 32$\times$32 images. Extension to higher resolutions requires validation that the energy features remain discriminative.

\section{Failure Analysis}
\label{app:failure_analysis}

We analyze cases where UFCOD underperforms to understand its limitations and guide practitioners.

\subsection{Near-OOD Detection}
Our most challenging case is CIFAR-10 vs CIFAR-100 (54.7\% AUROC). Both datasets contain natural object images with overlapping semantic categories (CIFAR-100 is a superset of CIFAR-10's classes). This suggests that energy features capture \emph{coarse} distributional differences (e.g., faces vs digits) but struggle with \emph{fine-grained} distinctions within similar visual domains. The diffusion model's score function provides similar guidance for semantically close images, resulting in overlapping energy distributions.

\subsection{Comparison with Successful Cases}
In contrast, CelebA vs SVHN achieves 99.99\% AUROC. The key difference lies in the semantic distance: faces and digits occupy fundamentally different regions of visual space, leading to dramatically different score function behaviors and well-separated energy distributions. This dichotomy suggests that UFCOD is most effective for detecting \emph{semantic} shifts (different object categories) rather than \emph{covariate} shifts within the same category.

\subsection{Density vs Geometry: Controlled Comparison}
To verify our hypothesis that geometry-based scoring outperforms density-based scoring, we conducted a controlled experiment comparing multiple scoring methods on the same diffusion model and features (Table~\ref{tab:density_vs_geometry}). Geometry-based methods achieve 85.0\% average AUROC vs 83.9\% for density-based methods, with our Soft-Min 2D achieving the best performance (86.9\%). This supports our core thesis that geometric features generalize better across domains.

\subsection{Practitioner Guidelines}
UFCOD is best suited for scenarios where ID and OOD domains are semantically distinct. For near-OOD detection (e.g., distinguishing between fine-grained categories), domain-specific fine-tuning or complementary methods may be necessary.

\begin{figure}[t]
\centering
\includegraphics[width=0.95\textwidth]{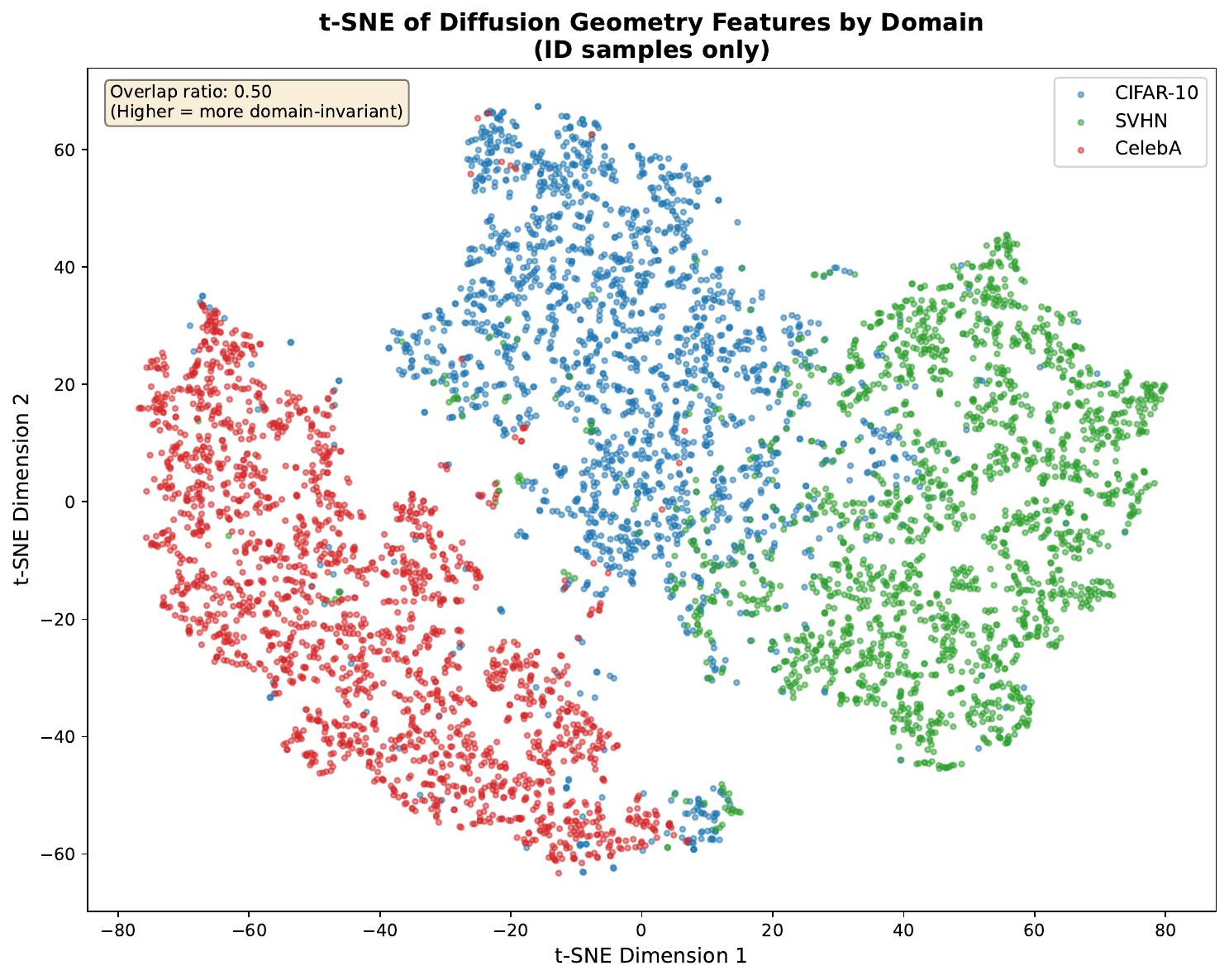}
\caption{\textbf{t-SNE visualization of diffusion geometry features across domains.} ID samples from CIFAR-10 (blue), SVHN (green), and CelebA (red) form distinct but partially overlapping clusters in the 2D energy feature space. The overlap ratio of 0.50 indicates that diffusion geometry features capture both domain-specific characteristics (enabling within-domain OOD detection) and shared geometric properties (enabling cross-domain transfer).}
\label{fig:tsne_domains}
\end{figure}

\section{Implementation Details}
\label{app:implementation}

\noindent \textbf{Diffusion Model.} We use a single DDPM~\cite{ho2020denoising} model pre-trained on CelebA (32$\times$32 resolution) as our universal feature extractor. This model uses a standard U-Net architecture with 4 resolution levels and is trained for 800k iterations. Crucially, this same model is used for all 12 ID-OOD pairs without any fine-tuning, demonstrating cross-domain transferability.

\noindent \textbf{Diffusion Energy Features.} For each input image $x_0$, we construct noised versions $x_t = \sqrt{\bar{\alpha}_t} x_0 + \sqrt{1-\bar{\alpha}_t} \epsilon$ at $T=10$ noise levels using a fixed noise sample $\epsilon \sim \mathcal{N}(0, \mathbf{I})$. At each level, we evaluate the pretrained denoiser to obtain $\epsilon_t = \epsilon_\theta(x_t, t)$. From this trajectory, we compute 2-dimensional energy features (Path Energy $f_1$ and Dynamics Energy $f_2$) as defined in Section~\ref{sec:method}. Feature extraction takes approximately 0.1s per image on a single GPU.

\noindent \textbf{Few-Shot Sampling.} We apply adaptive sample budgets based on dataset scale: $k=150$ for small datasets (CIFAR-10), $k=100$ for medium datasets (SVHN), and $k=80$ for large datasets (CelebA). Samples are selected using Facility Location~\cite{wei2015submodularity} coreset selection to maximize geometric coverage of the ID feature space.

\noindent \textbf{Detection Pipeline.} Given the selected reference set, we compute OOD scores using temperature-scaled proximity scoring ($T=0.5$, $k=10$) followed by Z-score calibration. The entire pipeline is training-free once features are extracted.

\noindent \textbf{Reproducibility.} All experiments use 5 random seeds for sample selection and report mean $\pm$ std. Experiments are conducted on a single NVIDIA RTX Pro 6000 GPU. Code and pretrained models will be released upon publication.

\section{Extended OOD Detection Results}
\label{app:extended_results}

We provide additional OOD detection results with baselines trained on different datasets.

\section{Additional Experimental Details}
\label{app:experiments}

\subsection{Score Calibration Analysis}
\label{app:calibration}

Table~\ref{tab:calibration_ablation} shows the effect of Z-score calibration on ID score distributions.

\begin{table}[t]
\centering
\caption{\textbf{Score calibration effect.} Z-score normalizes ID score distributions to $\mu \approx 0, \sigma \approx 1$ across domains.}
\label{tab:calibration_ablation}
\begin{tabular}{l|cc|cc}
\toprule
& \multicolumn{2}{c|}{\textbf{CIFAR-10 ID Scores}} & \multicolumn{2}{c}{\textbf{SVHN ID Scores}} \\
\textbf{Method} & $\mu$ & $\sigma$ & $\mu$ & $\sigma$ \\
\midrule
Raw Score & -2.17 & 0.78 & -2.45 & 0.84 \\
\textbf{Z-score (Ours)} & \textbf{0.08} & \textbf{0.86} & \textbf{-0.12} & \textbf{1.06} \\
\bottomrule
\end{tabular}
\end{table}

\begin{table}[t]
\centering
\small
\caption{\textbf{Empirical feature space statistics.} Mean and standard deviation of 2D energy features across datasets, validating our bounded feature space assumption.}
\label{tab:feature_stats}
\begin{tabular}{l|cc|cc}
\toprule
\textbf{Dataset} & \multicolumn{2}{c|}{\textbf{Path Energy ($f_1$)}} & \multicolumn{2}{c}{\textbf{Dynamics Energy ($f_2$)}} \\
& Mean & Std & Mean & Std \\
\midrule
CIFAR-10 (ID) & 4.21 & 1.82 & 2.15 & 0.94 \\
SVHN (ID) & 3.87 & 1.24 & 1.89 & 0.67 \\
CelebA (ID) & 5.62 & 2.31 & 2.78 & 1.12 \\
\midrule
\multicolumn{5}{l}{\textit{Feature space radius $R \approx 10$, supporting sample complexity bounds}} \\
\bottomrule
\end{tabular}
\end{table}

\subsection{Density vs Geometry Scoring Comparison}
\label{app:density_vs_geometry}

To isolate the effect of scoring approach from model choice, we conducted a controlled experiment comparing density-based and geometry-based scoring methods on the \textbf{same diffusion model and features}.

\textbf{Key Findings:}
(1) Geometry-based methods achieve 85.0\% average AUROC vs 83.9\% for density-based methods, supporting our hypothesis that geometric features generalize better.
(2 Our Soft-Min 2D method achieves the best performance (86.9\%), demonstrating the effectiveness of combining dimensionality reduction with soft scoring.
(3) The advantage is most pronounced for cross-domain pairs (C10$\to$SVHN, SVHN$\to$C10) where domain shift is significant.
(4) For near-OOD detection (C10$\to$C100), all methods struggle, confirming this as a fundamental limitation.

\begin{figure}[t]
\centering
\begin{subfigure}[b]{0.49\textwidth}
\centering
\includegraphics[width=\textwidth]{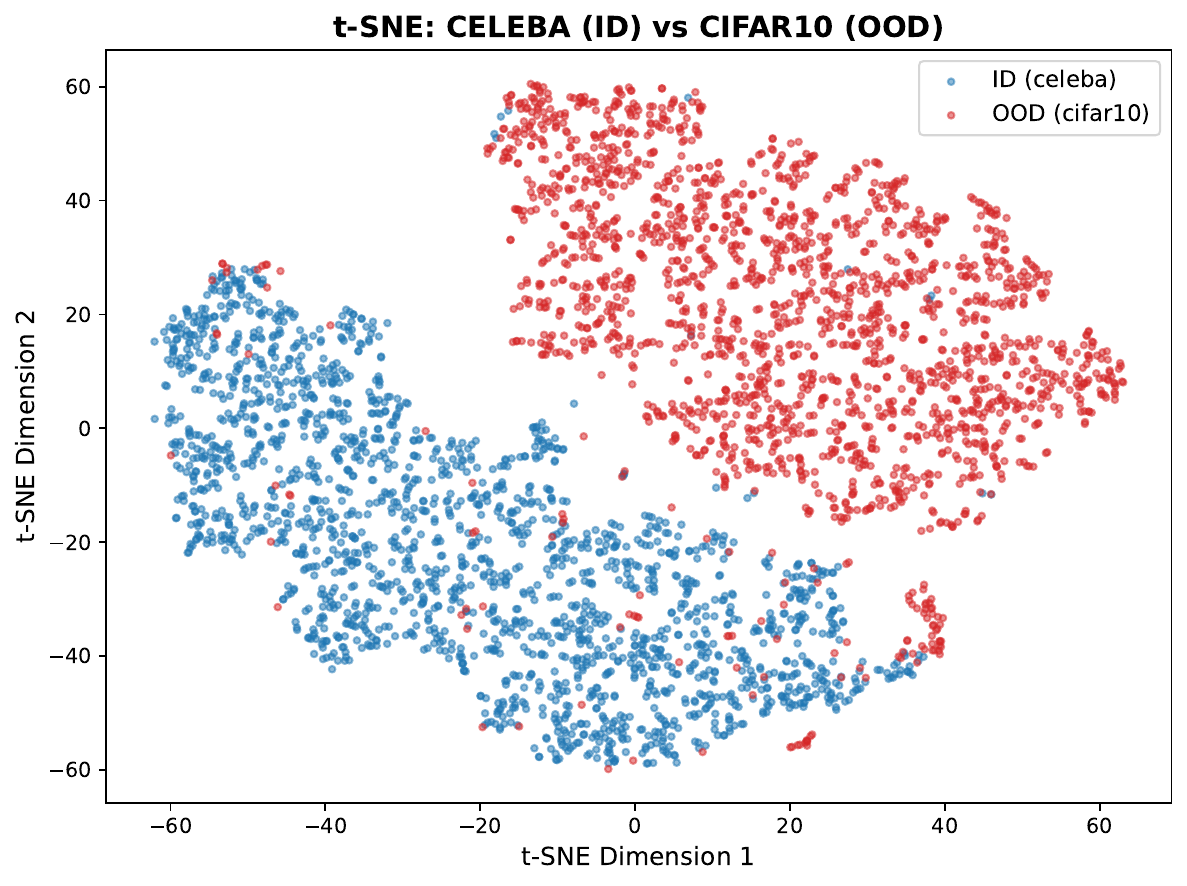}
\caption{CelebA (ID) vs CIFAR-10 (OOD)}
\label{fig:tsne_celeba_cifar10}
\end{subfigure}
\hfill
\begin{subfigure}[b]{0.49\textwidth}
\centering
\includegraphics[width=\textwidth]{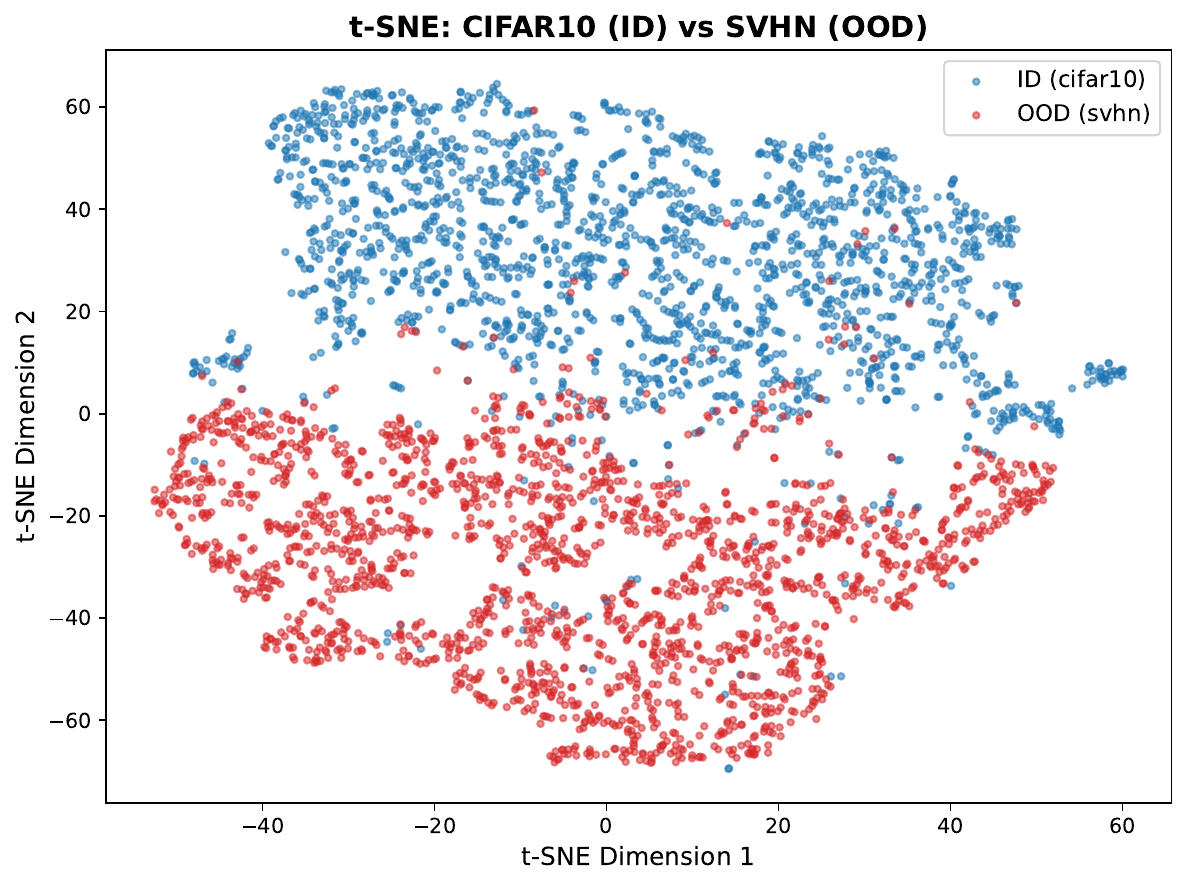}
\caption{CIFAR-10 (ID) vs SVHN (OOD)}
\label{fig:tsne_cifar10_svhn}
\end{subfigure}
\caption{\textbf{t-SNE visualization of ID vs OOD separation.} Diffusion geometry features achieve clear separation between ID (blue) and OOD (red) samples. (a) CelebA vs CIFAR-10: faces and natural objects occupy distinct regions. (b) CIFAR-10 vs SVHN: natural objects and digits are well-separated. Both pairs achieve $>$95\% AUROC with only 100 ID reference samples.}
\label{fig:tsne_id_ood}
\end{figure}

\subsection{Feature Space Statistics}
\label{app:feature_stats}

Table~\ref{tab:feature_stats} presents the empirical statistics of our 2D energy features across the three ID datasets. These statistics validate the bounded feature space assumption (Assumption 1 in Appendix~\ref{app:assumptions}) underlying our sample complexity analysis. The feature space radius $R \approx 10$ confirms that our covering number bounds are practically achievable, and the relatively low standard deviations indicate that ID samples form compact clusters in the energy feature space.

Figure~\ref{fig:tsne_domains} visualizes the distribution of ID samples from all three domains (CIFAR-10, SVHN, CelebA) in the diffusion geometry feature space using t-SNE. Several observations support our cross-domain detection framework: (1) Each domain forms a distinct cluster, confirming that diffusion geometry features capture domain-specific characteristics that enable OOD detection. (2) The clusters exhibit partial overlap (overlap ratio = 0.50), indicating that the features capture shared geometric properties of the diffusion process rather than purely domain-specific artifacts. (3) The moderate overlap explains why a single diffusion model trained on CelebA can generalize to detect OOD samples across semantically different domains: the feature space captures universal properties of how samples interact with the learned score function.

Figure~\ref{fig:tsne_id_ood} further visualizes the ID-OOD separation for two representative task pairs. In both cases, the diffusion geometry features achieve clear separation between ID and OOD samples in the 2D energy feature space: (a) CelebA (ID) vs CIFAR-10 (OOD) shows well-separated clusters with minimal overlap, corresponding to 99.5\% AUROC; (b) CIFAR-10 (ID) vs SVHN (OOD) exhibits similarly strong separation, achieving 95.1\% AUROC. These visualizations provide intuitive evidence for why simple proximity-based scoring in this low-dimensional feature space suffices for effective OOD detection.

\end{document}